%% file: DoE.tex
\documentclass{article}

     \usepackage[preprint]{neurips_2021}

\usepackage[utf8]{inputenc} % allow utf-8 input
\usepackage[T1]{fontenc}    % use 8-bit T1 fonts
\usepackage{hyperref}       % hyperlinks
\usepackage{url}            % simple URL typesetting
\usepackage{booktabs}       % professional-quality tables
\usepackage{amsfonts}       % blackboard math symbols
\usepackage{nicefrac}       % compact symbols for 1/2, etc.

\usepackage{amsmath,amsfonts,amsthm,amssymb}
\usepackage{setspace}
\usepackage{fullpage}
\usepackage{fancyhdr}
\usepackage{xspace}

\usepackage{titletoc}
\allowdisplaybreaks
\usepackage{verbatim}
\usepackage{lastpage}
\usepackage{extramarks}
\usepackage{natbib}
\usepackage{diagbox}

\usepackage{macros}
\input{notations/notations.tex}

\input{notations/math_commands.tex}

\input{notations/symbols.tex}

\usepackage{algorithm}
\usepackage{algorithmic}
\usepackage{graphicx}
\usepackage{subcaption}
\usepackage{wrapfig}
\usepackage{enumitem}
\input{notations/defs.tex}

\def\shownotes{1}  %set 1 to show author notes
\ifnum\shownotes=1
\newcommand{\authnote}[2]{[#1: #2]}
\else
\newcommand{\authnote}[2]{}
\fi

\allowdisplaybreaks
\renewcommand{\paragraph}[1]{\noindent\textbf{#1}}
\title{Design of Experiments \\
for Stochastic Contextual Linear Bandits}

\author{
Andrea Zanette* \\
Institute for Computational and Mathematical Engineering\\
Stanford University \\
Stanford, CA \\
\texttt{zanette@stanford.edu} \\
\AND
Kefan Dong* \\
Department of Computer Science \\
Stanford University \\
Stanford, CA \\
\texttt{kefandong@stanford.edu} \\
\And
Jonathan Lee* \\
Department of Computer Science \\
Stanford University \\
Stanford, CA \\
\texttt{jnl@stanford.edu} \\
\And
Emma Brunskill \\
Department of Computer Science \\
Stanford University \\
Stanford, CA \\
\texttt{ebrun@cs.stanford.edu} \\
}

\begin{document}
\maketitle
\begin{abstract}
In the stochastic linear contextual bandit setting there exist several minimax procedures for exploration with policies that are \emph{reactive} to the data being acquired. In practice, there can be a significant engineering overhead to deploy these algorithms, especially when the dataset is collected in a distributed fashion or when a human in the loop is needed to implement a different  policy.  Exploring with a single non-reactive policy is beneficial in such cases. Assuming some batch contexts are available, we design a single stochastic policy to collect a good dataset from which a near-optimal policy can be extracted. We present a theoretical analysis as well as numerical experiments on both synthetic and real-world datasets. 
\end{abstract}

\input{contents/1-introduction.tex}
\input{contents/3-related.tex}
\input{contents/2-preliminary.tex}
\input{contexts/4-mainresults.tex}

\input{contexts/5-experiments.tex}

\input{contexts/6-conclusion.tex}

\begin{small}
\bibliographystyle{apalike}
\bibliography{rl}
\end{small}

\input{contents/3b-checklist}
\newpage
\appendix
\input{contents/9-concentration.tex}
\input{contents/9-appendix-helper.tex}
\input{contents/9-appendix-exp.tex}
\end{document}

%% file: notations/notations.tex
\newcommand{\pbra}[1]{{\left( {#1} \right)}}
\newcommand{\bbra}[1]{{\left[ {#1} \right]}}

\newcommand{\maxtwo}[2]{\max \left( #1, #2 \right)}
\newcommand{\ind}[1]{\mathbb{I}\left[ #1 \right]}

\newcommand{\tildeO}{\widetilde{O}}

\newcommand{\thetas}{\theta^\star}

\newcommand{\lmin}{\lambda_{\mathrm{min}}}
\newcommand{\lmax}{\lambda_{\mathrm{max}}}

%% file: notations/math_commands.tex
\usepackage{amsmath,amsfonts,bm}

\def\1{\bm{1}}

\makeatletter
\newcommand{\ve}{\@ifnextchar\bgroup{\velong}{{\bm{e}}}}
\newcommand{\velong}[1]{{\bm{#1}}}
\makeatother

\DeclareMathAlphabet{\mathsfit}{\encodingdefault}{\sfdefault}{m}{sl}
\SetMathAlphabet{\mathsfit}{bold}{\encodingdefault}{\sfdefault}{bx}{n}

\def\calA{{\mathcal{A}}}

\def\calC{{\mathcal{C}}}
\def\calD{{\mathcal{D}}}
\def\calE{{\mathcal{E}}}

\def\calS{{\mathcal{S}}}

\newcommand{\mle}{\preccurlyeq}

%% file: notations/symbols.tex
\newcommand{\diag}{\mathrm{diag}}

\newcommand{\dotp}[2]{\left<#1, #2\right>}

\newcommand{\normtwo}[1]{\left\| #1 \right\|_2}

\newcommand{\normop}[1]{\left\| #1 \right\|_{\mathrm{op}}}

%% file: notations/defs.tex
\def\alphaone{\sqrt{2\ln 2|\StateSpace \times \ActionSpace| + \ln \frac{1}{\delta}}}
\def\alphatwo{2\sqrt{2d\ln 6 + \ln\frac{1}{\delta} }}

%% file: contents/1-introduction.tex
\section{Introduction}
Designers in many application areas can now deploy personalized decision policies.  
For example, recent work used manual classification of students to deploy semi-personalized literacy tips text messages which can be done easily by leveraging popular text messaging platforms. This work found significant improvements in young children's literacy scores (\citet{doss2019more}), but also noted the substantial variation in when and what type of text messages were most effective for different families (\citet{cortes2019behavioral,doss2019more}).

This setting and many others might substantially benefit from data-driven contextualized decision policies that optimize the desired expected outcome (such as literacy).  Online machine learning methods like multi-armed bandits and reinforcement learning, that  adaptively change interventions in response to outcomes in a closed loop process (see Figure~\ref{fig:RL}),   
may not yet be practical for such settings due to the expertise and infrastructure needed. However running an experiment with a fixed decision policy is likely to be both simple logistically (since such organizations already often design such decision policies by hand) and palatable,  % EB pick new word
since many areas (education, healthcare, social sciences) commonly deploy experiments across a few conditions to find the best approach. For these reasons, a key opportunity is to design \emph{static} or \emph{non-reactive} policies that can be used to gather data to identify optimal contextualized decision policies.

More generally, a fixed data collection strategy is practically desirable (1) whenever multiple agents collect data asynchronously and communication to update their policy is difficult or impossible, and (2) whenever changing the policy requires a significant overhead either in the engineering infrastructure or in the training of human personnel. Indeed several works limit the number of policy switches with minimal sample complexity impact \citep{han2020sequential,ruan2020linear,ren2020batched,bai2019provably,wang2021provably}. In this work, motivated by the above settings, we go further and look for a \emph{single}, \emph{non-reactive} policy for data collection.

\paragraph{Setting and goal}
We consider the linear stochastic contextual bandit setting where each context $s \in \StateSpace$ is sampled from a distribution $\mu$ and a context-dependent action set $a \in \ActionSpace_s$ is made available to the learner. The bandit instance is defined by a feature extractor $\phi(s,a)$ and some unknown parameter $\thetas$. Upon choosing an action $a \in \ActionSpace_s$, the linear reward function $r(s,a) = \phi(s,a)^\top\thetas + \eta$ is revealed to the learner corrupted by mean zero $1$-subGaussian noise $\eta$. Our goal is to construct an exploration policy $\pi_e$ to gather a dataset, such that after that dataset is gathered, we can extract a near optimal policy $\widehat \pi$ from it. 

Perhaps surprisingly, there has been relatively little work on this setting. %, or even the issue of data collection to identify the best decision policy for contextual bandits.  
Prior work on exploration to quickly identify a near-optimal policy focuses on best-arm identification using adaptive policies that react to the observed rewards \citep{soare2014best,tao2018best,jedra2020optimal} or design of experiments that produces a non-reactive policy for data collection \citep{kiefer1960equivalence, esfandiari2019regret,lattimore2020learning}; both lines of work assume that a \emph{single}, repeated context with unchanging action set is presented to the learner. In contrast we are interested in identifying near-optimal context-specific decision policies. The closest related work~\citep{ruan2020linear} investigates our task as a subtask for online regret learning, but  requires an amount of data that scales as $\Omega(d^{16})$, which is impractical in applications with even moderate $d$.

Without any apriori information about the problem, no algorithm can do much better than deploying a random policy, which can require an amount of data that scales exponentially in $d$, see \cref{sec:HardInstance}.
However, in many common applications, \emph{prior data about the context distribution $\mu(s)$} and the state--action feature representation $\phi$ is available. For example, an organization will often know information about its customers and specify the feature representation used for state--action spaces in advance of trying out such actions.\footnote{In other words, given a set of previously observed states $s_1,\ldots,s_M$, and a known state--action representation $\phi$, for any potential action $a$, we can compute the resulting representation  $\phi(s,a)$.} The initially available state contexts are referred to as \emph{offline} (state) contexts data.

Our algorithm leverages  historical context data  to enable data efficient design of experiments for stochastic linear contextual bandits. It uses offline context data $\C$ to design a non-reactive policy $\pi_e$
to collect new, \emph{online} data where reward feedback is observed (see Figure~\ref{fig:OD}), and uses the resulting dataset $\D'$ to learn a near-optimal on average decision policy $\pihat$ for future use. We highlight that the algorithm does not get to adjust the exploratory policy $\pi_e$ while the online data is being collected. 
\begin{figure}[t]
%\centering
\begin{subfigure}[b]{0.33\textwidth}
         \centering
         \includegraphics[width=0.5\linewidth]{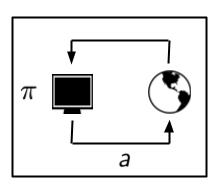} 
         \caption{Traditional online RL framework}
         \label{fig:RL}
     \end{subfigure}
     \hspace{1cm}
     \begin{subfigure}[b]{0.58\textwidth}
         \centering
     \includegraphics[width=\linewidth]{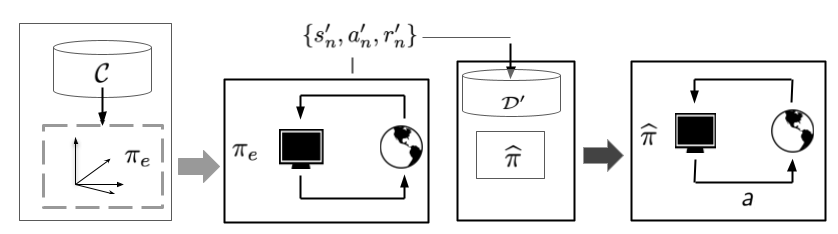} 
         \caption{Design of experiment with an online and offline component}
         \label{fig:OD}
     \end{subfigure}
\caption{Comparison between the traditional RL setting and design of experiments}
\label{fig::framework}
\vspace{-0.5cm}
\end{figure}

\paragraph{Contributions}
We make the following contributions.
\begin{itemize}[noitemsep,topsep=0pt,leftmargin=15pt]
    \item Using past state contexts only, we design a single, non-reactive policy to acquire online data that can be used to compute a context-dependent decision policy that is near-optimal in expectation across the contexts with high probability, for future use.   
    \item Our algorithm achieves the minimax lower bound $\Omega(\min\{d\log(|\StateSpace\times \ActionSpace |,d^2) \}/\epsilon^2)$ on the number of online samples (ignoring log and constants), while keeping the number of offline state contexts required polynomially small ($O(d^2/\epsilon^2)$ or $O(d^3/\epsilon^2)$). 
    \item Our experiments on a simulation and on a learning to rank Yahoo dataset show  our  strategic design of experiments approach can learn better decision policies with less exploration data compared to standard exploration strategies  that fail to exploit structure encoded in the historical contexts. 
\end{itemize}

%% file: contents/3-related.tex
\section{Related Work}
There is an extensive line of literature on linear bandits, with most papers focusing on online regret.  Research on  sample complexity results has focused primarily on adaptive algorithms for non-contextual bandits with a simple or vectored action space \citep{soare2014best,tao2018best,xu2018fully,fiez2019sequential,jedra2020optimal,degenne2020gamification}. 
Their sample complexity bounds depend on the suboptimality gap of the action set, and are therefore, instance-dependent. The design of experiments literature \citep{kiefer1960equivalence, esfandiari2019regret,lattimore2020learning,ruan2020linear} has also focused on non-contextual bandits, but in designing a non-adaptive policy such that the resulting dataset can be used to identify an optimal arm with high probability. 

A few recent papers consider learning the optimal policy for contextual bandits in the pure exploration setting \citep{deshmukh2018simple,ruan2020linear}. \citet{deshmukh2018simple} propose an algorithm that yields an online reactive policy for when the reward is generated by a RKHS function on the input space: in contrast we focus on non-adaptive policies for stochastic linear contextual bandits. 

The most related paper to this work is by  \citet{ruan2020linear}. The main focus of \citet{ruan2020linear} is to minimize the number of policy switches in the online phase while keeping the expected regret optimal. Directly applying Theorem 6 of \citet{ruan2020linear} yields a $\tildeO(d^2/\epsilon^2)$ online sample complexity and $\tildeO(d^{16})$ offline sample complexity. 
Note that the offline sample complexity is independent of $\epsilon$. As a trade-off, \citet{ruan2020linear} have a highly suboptimal dependence on $d$, which leads to burn-in phase too large to be practically useful; their algorithm also suffers from a high computational complexity. However, note their work is focused on a different objective than what we have here.

%% file: contents/2-preliminary.tex
\section{Setup}
We consider the stochastic linear contextual bandit model with stochastic contexts. A bandit instance is characterized by a tuple $\left<\calS,\calA,\mu,r\right>$ where $\calS$ is the context space and $\mu$ is the context distribution. For a context $s\in\calS$, the action space is denoted by $\calA_s$. The feature map $\phi: (s,a) \mapsto \phi(s,a)\in\R^d$ is assumed to be known and defines the linear reward model $r(s,a)=\phi(s,a)^\top \thetas+\eta$ for some $\theta^\star \in \R^d$ parameter and some mean-zero  1-subgaussian random variable $\eta$.

We occasionally use the $\widetilde O$ notation to suppress $\polylog$ factors of the input parameters $d,\lambda,\frac{1}{\delta}$. We write $f\lesssim g$ if $f=O(g)$, and $f\lessapprox g$ if $f=\tildeO(g).$ For a positive semi-definite matrix $\Sigma\in\R^{d\times d}$ and a vector $x\in\R^{d}$, let $\norm{x}{\Sigma}=\sqrt{x^\top \Sigma x}$. For two symmetric matrices $A,B$, we say $A\mle B$ if $B-A$ is positive semi-definite. 
 
Our analysis consists of an offline and online component. We often add $'$ to indicate the online quantities, e.g., we denote with $s_1,s_2,\dots$ the offline contexts and with $s'_1,s'_2,\dots$ the online contexts. 
\subsection{Objective and Error Decomposition}
As depicted in Figure~\ref{fig::framework}, our approach is to leverage offline state contexts $\C = \{s_1,\dots,s_M \}$ where $s_m \sim \mu$ to design a single stochastic policy $\pi_e$ to acquire data from an online data stream $\C' = \{ s'_1,\dots,s'_N\}$. 
This generates a dataset $\D' = \{(s'_n,a'_n,r'_n)\}_{n=1,\dots,N}$ where $s'_n \sim \mu, a'_n \sim \pi_e(s_n')$ and $r'_n = r(s'_n,a'_n)$. Using $\D'$, the least-square predictor $\widehat \theta$ and the corresponding greedy policy $\widehat \pi$ can be extracted
\begin{align}
\label{main:eqn:thetahat}
	\widehat \theta = \big(\Sigma'_{N}\big)^{-1}\sum_{i=1}^N \phi(s'_n,a'_n)r'_n, \qquad \qquad \widehat \pi(s) = \argmax_{a \in \ActionSpace_s}\phi(s,a)^\top \widehat \theta
\end{align}
where $\Sigma'_{N} = \Big(\sum_{n=1}^{N} \phi(s'_n,a'_n)\phi(s'_n,a'_n)^\top + \lambda_{reg} I\Big)$ is the empirical cumulative covariance matrix with regularization level $\lambda_{reg}$.
The quality of the dataset $\D'$ (and of the whole two-step procedure) is measured by the suboptimality of the extracted policy $\widehat \pi$ obtained after data collection: 
\begin{align}
	\E_{s\sim \mu}[\max_a \phi(s,a)^\top \thetas -\phi(s,\widehat \pi(s))^\top \thetas].\label{eqn:regret}
\end{align}
Note that Eq.~\eqref{eqn:regret} measures the expectation over the contexts of the suboptimality between the resulting decision policy $\widehat{\pi}$ and the optimal policy. This is a looser criteria than a maximum norm bound which evaluates the error over any possible context $s$: in general this latter error may  not be easily reduced if certain directions in feature space are rarely available.
	
A related objective is to minimize the maximum prediction error on the linear bandit instance 
\begin{eqnarray}
	\E_{s\sim \mu} \max_{a}|\phi(s,a)^\top (\theta^\star - \widehat \theta )| &\leq& \E_{s\sim \mu} \max_{a} \| \phi(s,a) \|_{(\Sigma'_N)^{-1}} \|\theta^\star - \widehat \theta \|_{\Sigma'_N} \\
	& \leq& \sqrt{\beta}\E_{s\sim \mu} \max_{a}\| \phi(s,a) \|_{(\Sigma'_N)^{-1}}.
	\label{main:eqn:uncertainty_decomposition}
\end{eqnarray} 
In Eq.~\eqref{main:eqn:uncertainty_decomposition}, the least square error $\|\theta^\star - \widehat \theta \|_{\Sigma'_N} \leq \sqrt{\beta}$ comes from least square regression analyses \citep{lattimore2020bandit}: with probability at least $1-\delta$ we have
\begin{align}
\label{main:eqn:beta}
\| \widehat \theta - \theta^\star \|_{\Sigma'_N}\leq 	\sqrt{\beta}=\min\Big\{ \underbrace{\alphaone \vphantom{\alphatwo}}_{\text{Small $\StateSpace \times\ActionSpace$}},\underbrace{\alphatwo}_{\text{Large $\StateSpace \times\ActionSpace$}}
		\Big\} + \underbrace{\sqrt{\lambda_{reg}}\| \theta^\star\|_2 \vphantom{\alphatwo}}_{\text{Regularization Effect}}.
\end{align}
The above expression assumes that the state-action-rewards $(s_n',a_n',r_n')$ are drawn i.i.d.  from a fixed distribution. This is satisfied in our setting as the data-collection policy $\pi_e$ is non-reactive to the online data. The parameter $\beta$ governs the sample complexity as a function of the size of the state-action space and also highlights the impact of the regularization bias. 

Small predictive error (Eq.~\eqref{main:eqn:uncertainty_decomposition}) can be used to bound the suboptimal gap of the greedy policy (Eq.~\eqref{eqn:regret}). Therefore to obtain good performance, 
it is sufficient to bound  
$\E_{s\sim \mu} \max_{a}\| \phi(s,a) \|_{(\Sigma'_N)^{-1}}$.  This can be achieved by designing an appropriate sampling policy $\pi_e$ to yield a set of $(s_n',a_n',r_n')$ tuples whose  empirical cumulative covariance matrix $\Sigma'_N$ is as `large' as possible.

%% file: contexts/4-mainresults.tex
\section{Algorithms}
\paragraph{Reward-free \textsc{LinUCB}} First, assume it is acceptable to have an algorithm that updates its policy reactively. In order to reduce $\E_{s \sim \mu}\max_{a}\| \phi(s,a) \|_{(\Sigma'_n)^{-1}}$ we could use an algorithm that, every time a context $s \sim \mu$ is observed, chooses the action $\argmax_{a \in \ActionSpace_s} \| \phi(s,a)\|_{(\Sigma'_n)^{-1}}$ where the norm $\| \phi(s,a)\|_{(\Sigma'_n)^{-1}}$ that represents the uncertainty is highest (cf. \cref{main:eqn:uncertainty_decomposition}). This corresponds to running the \textsc{LinUCB} algorithm \citep{Abbasi11} with the empirical reward function set to zero. 
One can show that after $\approx d^2/\epsilon^2$ iterations the uncertainty $\sqrt{\beta}\E_{s \sim \mu}\max_{a}\| \phi(s,a) \|_{(\Sigma'_n)^{-1}} \leq \epsilon$; if the algorithm stores the observed reward $r(s,a)$ in every visited context $s$ and chosen action $a$, the greedy policy that can be extracted from this dataset (of size $\approx d^2/\epsilon^2$) is $\epsilon$-optimal (cf. Eqs.~\eqref{eqn:regret},\eqref{main:eqn:uncertainty_decomposition}), as desired. 

Unfortunately we cannot run this reward-free algorithm online, as its policy is \emph{reactive} to the online stream of observed online contexts $s$ and selected actions $a$, while we want a non-reactive policy. 

\begin{minipage}{0.55\textwidth}
\begin{algorithm}[H]
\floatname{algorithm}{Algorithm}
\begin{algorithmic}[1]
\caption{\footnotesize \textsc{Planner} (Reward-free \textsc{LinUCB})}
\label{alg:Planner}
\STATE \textbf{Input}: Contexts $\C = \{s_1,\dots,s_{M} \}$, reg. $\lambda_{reg}$
\STATE $\Sigma_1 = \lambda_{reg} I$
\STATE $m = 1$
\FOR{$m = 1,2,\dots M$}
\IF{$\det(\Sigma_m) > 2\det(\Sigma_{\underline m})$ or $m = 1$}
\STATE $\underline m \leftarrow  m$
\STATE $\Sigma_{\underline m} \leftarrow \Sigma_{m}$
\ENDIF
\STATE Define $\pi_m : s \mapsto \argmax_{a \in \ActionSpace_s} \| \phi(s,a)\|_{\Sigma^{-1}_{\underline m}}$
\STATE $\Sigma_{m+1} = \Sigma_{m} + \alpha\phi_m\phi_m^\top; \ \phi_m = \phi(s_m,\pi_m(s_m))$
\ENDFOR
\RETURN policy mixture $\pi_{mix}$ of $\{\pi_1,\dots,\pi_M\}$
\end{algorithmic}
\end{algorithm}
\end{minipage}
\begin{minipage}[H]{0.47\textwidth}
\begin{algorithm}[H]
%\floatname{algorithm}{Algorithm}
\begin{algorithmic}[1]
\caption{\footnotesize \textsc{Sampler}}
\label{alg:Sampler}
\STATE \textbf{Input}: $\pi_{mix} = \{\pi_1,\dots,\pi_{M} \}$, reg. $\lambda_{reg}$
\STATE Set $\mathcal D' = \emptyset$
\FOR{$n = 1,2,\dots N$}
\STATE Receive context $s'_n \sim \mu$
\STATE Sample $m \in [M]$ uniformly at random
\STATE Select action $a'_n = \pi_m(s'_n)$
\STATE Receive feedback reward $r'_n$
\STATE Store feedback $\mathcal D' = \mathcal D' \cup \{s'_n,a'_n,r'_n\}$ 
\ENDFOR
\RETURN dataset $\mathcal D'$
\end{algorithmic}
\end{algorithm}
\end{minipage}

Our algorithm leverages this idea and consists of two subroutines: 1) the \emph{planner} (Alg.~\ref{alg:Planner}) which operates on offline contexts and identifies a mixture policy $\pi_{mix}$ (this is the exploratory policy $\pi_e$ mentioned in the introduction) 2) the \emph{sampler} (Alg.~\ref{alg:Sampler}) which runs $\pi_{mix}$ online to finally gather the dataset. This way, $\pi_{mix}$ is non-reactive to the online data.

\paragraph{Planner}
The purpose of the planner is to use past contexts to compute the exploratory policy.
The planner runs the reward-free version of \textsc{LinUCB} on the offline context dataset $\C$ as described earlier in this section. This way, the planner selects the action $a$ in the current offline context $s_m$ that maximizes the uncertainty encoded in $\| \phi(s_m,a)\|_{\Sigma^{-1}_{\underline m}}$ where $\Sigma_{\underline m}$ is a scaled, regularized, cumulative covariance over the contexts parsed so far and the actions selected. Note this procedure is possible since the state--action function $\phi(s_m,a)$ is assumed to be known for any input $(s_m,a)$ pair, and no actual rewards are observed. 
The doubling schedule yields a short descriptor for the planner's policies $\{ \pi_1,\dots,\pi_M\}$. The variable $\underline m$ indicates the last doubling update before iteration $m$.

A key choice is the parameter $\alpha \in (0,1]$ in the cumulative covariance matrix update. The rationale is that when $\alpha < 1$ each  rank-one update $\phi_m\phi_m^\top$ to the cumulative covariance gets \emph{discounted}. The smaller $\alpha$ is, the more offline samples the planner needs to get to a sufficiently positive definite covariance matrix $\Sigma_M$.  This choice effectively averages the updates and increases the estimation accuracy of the planner's covariance with respect to its conditional expectation.

\paragraph{Sampler} Upon termination, the planner identifies a sequence of policies $\pi_1,\dots,\pi_M$. Now consider the average policy $\pi_{mix}$: every time an action is needed, $\pi_{mix}$ samples one index $m \in [M]$ uniformly at random and plays $\pi_m$. This is the policy that the sampler (Alg.~\ref{alg:Sampler}) implements for $N = \alpha M \leq M$ fresh online contexts. 

Upon playing $\pi_m$ in state $s'_n$, the corresponding reward is observed and the tuple ($\{s'_n,a'_n,r'_n\}$) is stored. Since $\pi_{mix}$ is the average policy played by the planner, we expect that running $\pi_{mix}$ on the online dataset produces a covariance matrix $\Sigma'_N$ similar to the planner's $\Sigma_M$. This means the sampler acquires the same information that the planner would have acquired in presence of a reward signal. Since the planner is the reward-free \textsc{LinUCB} algorithm, we expect its policy to efficiently reduce our uncertainty over the reward parameters and learn a near optimal policy; our analysis will make this intuition precise. Note the sampler's policy is non-reactive to the online stream of data.

\section{Main Result}

\good{AZ: need to do high-level checks on the below} 
\good{AZ: need to make the below more precise in terms of setting the parameters and the log dependencies}
%\knote{here I suggest we set $\lambda_{reg}\in(\ln(d/\delta),d].$}
\begin{theorem}
\label{thm:main}
Consider running Alg.~\ref{alg:Planner} for $M = \widetilde\Omega (\frac{d^2 \beta}{ \lambda_{reg}\epsilon^2})$ iterations and	Alg.~\ref{alg:Sampler} for $
N = \widetilde\Omega( \frac{d\beta}{\epsilon^2})$ iterations with regularization $\lambda_{reg}\in (\Omega(\ln (d/\delta)),d]$. Let $\widehat \theta, \widehat \pi$ be as in Eq.~\eqref{main:eqn:thetahat}. For any $\epsilon\le 1$, with probability at least $1-\delta$ the expected maximum uncertainty
\begin{align}
	\E_{s \sim \mu}\max_{a \in \ActionSpace_s} |\phi(s,a)^\top(\theta^\star - \widehat\theta)|\leq \epsilon
\end{align}
and the suboptimality of the greedy policy $\pihat$ satisfies
\begin{align}\label{equ:thm1-2}
\E_{s \sim \mu}\max_{a \in \ActionSpace_s}  \left( \phi(s,a)-\phi(s,\pihat(s)) \right)^\top\theta^\star \leq 2\epsilon.
\end{align}
\end{theorem}
In Table~\ref{main:table:MainTable} we  instantiate the bounds from Theorem~\ref{thm:main} in different settings, ignoring constants and log terms. This highlights different tradeoffs between the regularization $\lambda_{reg}$ and the number of offline contexts ($M = N /\alpha$) needed to achieve the online minimax sample complexity lower bound $\approx \min \{ d, \ln |\StateSpace \times \ActionSpace | \} \times d/\epsilon^2$ \citep{chu2011contextual,abbasi2011improved}.

\begin{figure}[tb]
	\centering
\begin{tabular}{ |p{2cm}|p{0.5cm}||p{2.5cm}|p{2.5cm}|p{2.5cm}|  }
 \hline
 \multicolumn{5}{|c|}{Sample Complexity Bounds} \\
 \hline
 & & Small $\StateSpace\times\ActionSpace$ & Large $\StateSpace\times\ActionSpace$ & Large $\StateSpace\times\ActionSpace$ \\
 \hline
Offline Data & $M$     & $d^3/\epsilon^2$& $d^3/\epsilon^2$  &$d^2/\epsilon^2$\\
Online Data & $N$  & $(d\ln|\StateSpace\times\ActionSpace |)/{\epsilon^2}$   & $d^2/\epsilon^2$ & ${d^2}/{\epsilon^2}$ \\
Regularization & $\lambda_{reg}$     & $1$ & $1$ & $d$ \\ 
\hline
\end{tabular}
\captionof{table}{Sample complexity bounds (ignoring constants and log terms) to obtain an  $\epsilon$-optimal  policy}
\label{main:table:MainTable}
\end{figure}
In the large action regime $\ln |\StateSpace \times \ActionSpace| \gtrapprox d$, a regularization level $\lambda_{reg} \approx d$ gives the optimal online sample complexity $N\approx d^2/\epsilon^2$ while requiring a context dataset of the same size ($M\approx d^2/\epsilon^2$, by choosing $\alpha = 1$). 

More often, a lower level of regularization can be preferable to introduce less bias. For example $\lambda_{reg} = 1$ is a common choice in linear bandits \citep{abbasi2011improved,chu2011contextual}.  When the cumulative covariance matrix is less regularized, we may need to compensate with additional offline data to ensure the covariance matrices are accurately estimated, which is achieved by setting $M = N/\alpha \approx dN$ (i.e., $\alpha \approx 1/d$). In this way, additional samples are collected by the planner to maintain its covariance estimation accuracy (and hence its planning accuracy) despite the lower regularization. Interestingly in our experiments the algorithm performed well numerically even when $\lambda_{reg} < 1$ and $\alpha = 1$.

More generally, the amount of regularization should be set following the classical bias-variance tradeoff; its choice is mostly a statistical learning question, and its optimal value is normally problem dependent. As a result, the problem dependent  tradeoff between different values of $\lambda_{reg}$ is not reflected in our minimax analyses but can be appreciated in our numerical experiments on the real world dataset. 

\section{Proof Sketch}

We give a brief proof sketch where we ignore constants and log factors; the full analysis can be found in the  appendix. First we introduce some notation. 
We define the scaled cumulative covariance $\Sigma_M$
 at step $m$ for the planner and the cumulative covariance $\Sigma'_n$ for the sampler at step $n$:
 \begin{align}
\label{main:eqn:covariances}
\Sigma_{m} = \alpha \sum_{j=1}^{m-1} \phi(s_j,a_j)\phi(s_j,a_j)^\top + \lambda_{reg} I,  \qquad\Sigma'_{n} & = \sum_{j=1}^{n-1} \phi(s'_j,a'_j)\phi(s'_j,a'_j)^\top + \lambda_{reg} I
\end{align}
where the planner's $j$ actions is  $a_j = \pi_{j}(s_j)$ and likewise the sampler's $j$ actions is $a'_j \sim \pi_{mix}(s'_j)$.

\paragraph{Uncertainty Stochastic Process}
We represent the value of the uncertainty though a stochastic process.
Let us define the filtration for the planner  $\F_{m} = \sigma(s_1,\dots,s_{m-1})$ at stage $m$ and for the sampler $\F'_{n} = \sigma(s_1,\dots,s_{M},s'_1,\dots,s'_{n-1})$ at stage $n$ to represent the amount of information available.
Let us also define the observed uncertainty $U_m$ for the planner and likewise $U'_n$ for the sampler:
\begin{align}
U'_m \defeq \max_{a \in \ActionSpace_{s_m}} \| \phi(s_m,a) \|_{(\Sigma_m)^{-1}}, \qquad 
U'_n \defeq \max_{a \in \ActionSpace_{s_n}} \| \phi(s_n,a) \|_{(\Sigma'_n)^{-1}}.
\end{align}
We ultimately want to bound $\E'_n U'_n = \E[U'_n \mid \F_n] = \E_{s\sim \mu} \max_{a}\| \phi(s,a) \|_{(\Sigma'_n)^{-1}}$ when $n = N$, i.e., the average uncertainty at the end, see Section~\ref{main:eqn:uncertainty_decomposition}. Likewise, let us define $\E_m U_m = \E[U_m \mid \F_m] = \E_{s\sim \mu} \max_{a}\| \phi(s,a) \|_{(\Sigma_m)^{-1}}$.

We show that $\E'_N U'_N$ is minimized in two steps: 1) we show that the planner's uncertainty $\E_M U_M$ can be bounded 2) since the sampler implements the planner's average policy, $\E'_N U'_N$ cannot be too far from $\E_M U_M$.

 We highlight that $M \geq N$ in general, i.e., the stochastic processes proceed a different speed. In fact, the planner needs more data than the sampler ($M \geq N$) as the planner's policy is reactive to the offline contexts; this forces us to use more data-intensive  concentration inequalities for the planner. 

\subsection{Offline Uncertainty}
The next lemma is formally presented in Lemma~\ref{lem:unbound} and Lemma~\ref{lem:SumOfUncertainties} in the appendix and examines the reduction in the planner's expected uncertainty $\E_M U_M$.
\begin{lemma}[Offline Expected Uncertainty] 
\label{main:lem:unbound}
With high probability we have
$$
\E_M U_M\leq \frac{1}{M}\sum_{m=1}^M \E_m U_m  \lessapprox \frac{1}{M} \sum_{m=1}^M U_m   \lessapprox \sqrt{\frac{d}{\alpha M}}.
$$
\end{lemma}
\begin{proof}[Proof sketch]
We start by modifing a result from the classical bandit literature  (e.g., \cite{Abbasi11}) to bound the average realized uncertainty $U_m$ while accounting for the extra $\alpha$ scaling factor contained in the covariance matrix $\Sigma_M$:
$$ \frac{1}{M}\sum_{m=1}^M U_m = \frac{1}{M}\sum_{m=1}^M \|\phi(s_m,a_m) \|_{\Sigma^{-1}_m} \lessapprox \sqrt{\frac{d}{\alpha M}}. $$ 
The parameter $\alpha$ on the right hand side above arises from its inclusion in the  definition of cumulative covariance Eq.~\ref{main:eqn:covariances}. This justifies the last inequality in the lemma's statement.

While we bounded $\frac{1}{M}\sum_{m=1}^M U_m$, we need to bound the average conditional expectations $\frac{1}{M}\sum_{m=1}^M \E_m U_m$. Bernstein's inequality would bound the actual deviation $\sum_{m=1}^M U_m$ by the conditional variances (or expectations $\sum_{m=1}^M \E_m U_m$ given our boundness assumptions); here we need the opposite, and so we `reverse' the inequality in \fullref{thm:ReverseBernstein} to claim 
$\sum_{m=1}^{M} \E_m U_m \lessapprox \sum_{m=1}^{M} U_m $ with high probability to conclude.
\end{proof}

\subsection{Online Uncertainty}
In the prior section we showed that the planner would be successful in reducing its final uncertainty $\E_M U_M$ if it was collecting reward information; moreover,  its sequence of policies $\pi_1,\dots,\pi_M$ only depends on the observed contexts. Because of this, if the sampler runs the planner's average policy $\pi_{mix}$, we would expect a similar reduction in the uncertainty $\E'_N U'_N$ after $N = \alpha M$ iterations (the $\alpha$ factor simply arises because the planner's information are discounted by $\alpha$). The argument is formalized in \fullref{lem:OfflineOnline} in the appendix, which we preview here. We let  $K$ be the number of policy switches by the sampler which is bounded by $K \lessapprox d$ in \fullref{lem:Switches} in appendix.

\begin{lemma}[Relations between  Offline and Online Uncertainty]
If $\lambda_{reg} = \Omega(\ln\frac{d}{\delta})$ and $
	M = \Omega\(\frac{KN}{\lambda_{reg}} \ln\frac{dNK}{\lambda_{reg}\delta}\)
$, upon termination of  Alg.~\ref{alg:Planner} and \ref{alg:Sampler} it holds with high probability that
\begin{align*}
\E'_N U'_{N} \lesssim \E_M U_{M}.
\end{align*}
\end{lemma}
\begin{proof}

Let $d_1,\dots,d_M$ be the conditional distributions of the feature vectors sampled at timesteps $1,\dots,M$ in Alg.~\ref{alg:Planner} after the algorithm has terminated.
Conditioned on $\mathcal F_M$, the $d_i$'s are non-random. Then the conditional expected covariance matrix given $\F_M$ can be defined as
\begin{align}
\overline \Sigma & = \alpha \sum_{i=1}^M \E_{\phi \sim d_i} \phi\phi^\top + \lambda_{reg} I
\end{align}
Now, our argument relies on some matrix concentration inequalities which hold if 
\begin{align}
\label{main:eqn:preconditions}
	\lambda_{reg} = \Omega(\ln\frac{d}{\delta}), \qquad \frac{1}{\alpha} = \Omega\(\frac{K}{\lambda_{reg}} \ln\frac{dNK}{\lambda_{reg}\delta}\).
\end{align}
Precisely, from \fullref{lem:SigmaOffline} and \fullref{lem:SigmaOnline} with high probability we have that
\begin{align}
\label{main:eqn:covok}
\forall x, \| x \|_2\leq 1: \qquad \qquad  
	\| x \|_{(\Sigma_N')^{-1}} \lesssim\;   \| x \|_{\overline \Sigma^{-1}}
	\qquad \text{and} \qquad \| x \|_{\overline \Sigma^{-1}} \lesssim \;   \| x \|_{\Sigma^{-1}_M}.
\end{align}
In words, we can relate the planner's (random) scaled covariance $\Sigma_M$ to its conditional expectation $\overline \Sigma$, with the guarantee that it won't be very different from the sampler's covariance $\Sigma'_{N}$ (which implements the planner's policy). These concentration inequalities are key to the proof  and are proved in the appendix. 

Define the policy maximizing the online uncertainty 
$
\pi'_{n}(s) = \argmax_{a \in \ActionSpace_s} \| \phi(s,a) \|_{(\Sigma'_n)^{-1}}.
$
Under the event in Eq.~\eqref{main:eqn:covok}  
we can write 
\begin{align*}
\E'_N U'_N & \defeq \E_{s \sim \mu} \max_{a \in \ActionSpace_s} \| \phi(s,a) \|_{(\Sigma'_N)^{-1}} = \E_{s \sim \mu}\| \phi(s,\pi'_N(s)) \|_{(\Sigma'_N)^{-1}}  \lesssim \E_{s \sim \mu}\| \phi(s,\pi'_N(s)) \|_{\overline \Sigma^{-1}} \\
& \lesssim \E_{s \sim \mu}\| \phi(s,\pi'_N(s)) \|_{\Sigma_M^{-1}}  \leq \E_{s \sim \mu} \max_{a} \| \phi(s,a) \|_{\Sigma_{M}^{-1}}  = \E_M U_M.
\end{align*}
\end{proof}

\subsection{Conclusion and Tradeoffs}
We can now tune the parameters $\alpha,\lambda_{reg}$. Starting from
Eq.~\eqref{main:eqn:uncertainty_decomposition},
we can write
\begin{align}
	\E_{s\sim \mu}\max_a\phi(s,a)^\top (\theta^\star-\widehat\theta) \leq \sqrt{\beta} \E'_NU'_N \lesssim\sqrt{\beta} \E_M U_M \lesssim \sqrt{\frac{\beta d}{\alpha M}} = \sqrt{\frac{\beta d}{N}}.
\end{align}
Thus, the rate-optimal online sample complexity $N \approx \frac{\beta d}{\epsilon^2}$ always suffices. However, we need a different number of offline contexts depending on the regularization that we want to adopt. Since $K \lessapprox d$, the reader can verify that if $\lambda_{reg} \approx d$ then the preconditions in Eq.~\eqref{main:eqn:preconditions} are verified with $M \approx N$, so $\alpha \approx 1$. If a lower value for the regularization $\lambda \approx 1$ is desired, $M \approx d N$ would be required (thus $\alpha \approx \frac{1}{d}$).

The reason why we might need $M \geq N$ (i.e., $\alpha \leq 1$) is that the planner's policy is adaptive to the observed offline contexts. This means that the planner's cumulative covariance $\Sigma_M$ is not a sum of i.i.d. rank-one random matrices; in this case, we need either a more data-hungry concentration inequality or heavier regularization, leading to the the precondtion on $\alpha$ in Eq.~\eqref{main:eqn:preconditions}.

%% file: contexts/5-experiments.tex
\section{Experiments}

In this section, we study the empirical properties of the planner and sampler in both a synthetic dataset and a real-world dataset from the Yahoo! Learning to Rank challenge \citep{chapelle2011yahoo}. The objective of these experiments is to compare the performance of the planner-sampler pair with alternatives and analyze its sensitivity to regularization.

\subsection{Synthetic dataset}

\begin{figure}[h]
\vspace{-0.5cm}
\begin{subfigure}{.5\textwidth}
\centering
\includegraphics[width=5cm]{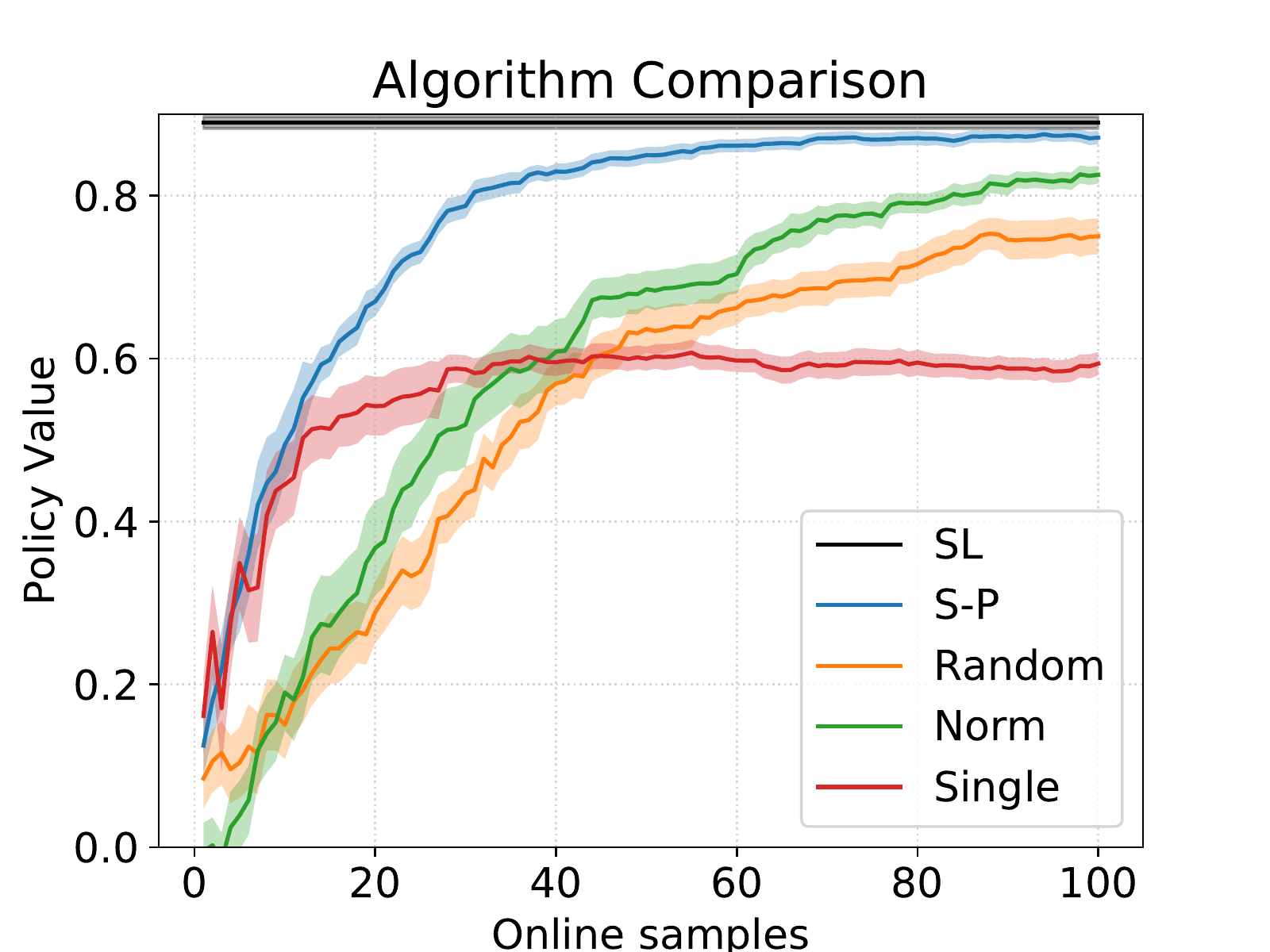} 
\caption{}
\end{subfigure}
\begin{subfigure}{.5\textwidth}
\centering
\includegraphics[width=5cm]{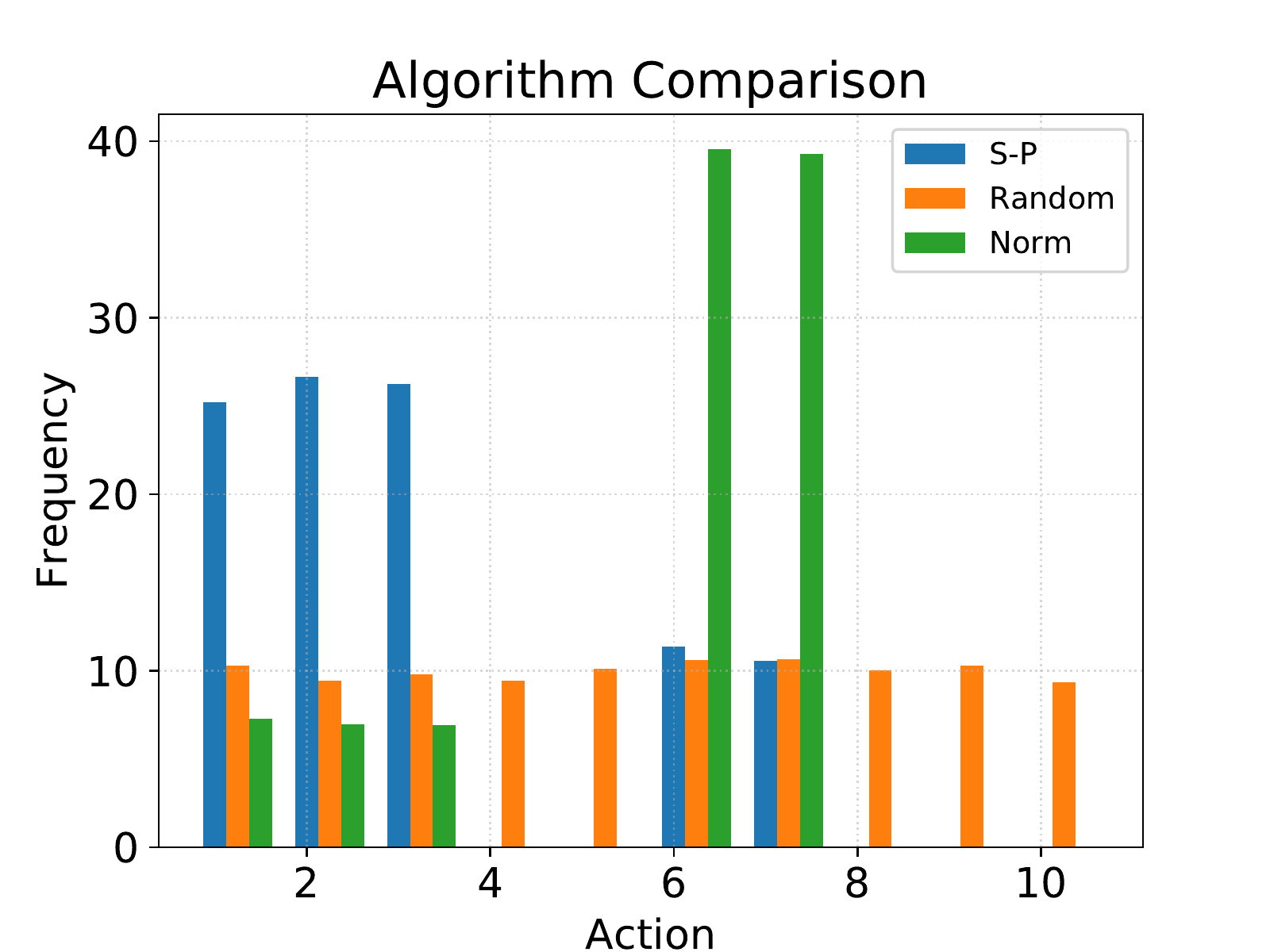}
\caption{}
\end{subfigure}%
    \caption{Synthetic Setting. In the left figure, the policy value performance of the sampler-planner (S-P) is compared with a random algorithm (Random), a largest norm algorithm (Norm), and an algorithm that always chooses the same single action (Single). Supervised learning (SL) represents an approximate upper bound; however, we don't expect the algorithms to be able to reach the SL performance as they all receive bandit feedback compared to the full feedback available to SL (10,000 samples). In the right figure, the empirical distributions over the actions chosen by the algorithms are compared, with the exception of Single, which always chooses action 1. }\label{fig::synth}
    \vspace{-3mm}
\end{figure}

We first conduct a synthetic validation experiment in order to verify that the algorithms behave consistently with the intuition and theory developed in the previous sections. Towards this end, we constructed a simple linear contextual bandit problem with $d = 20$ and $\mathcal A = \{ 1,\ldots, 10\}$. Each context $s \in \mathcal S$ is associated with features vectors $\{ \phi(x, a)\}_{a \in \mathcal A}$. 
Each context belongs to one of three discrete categories with equal probability. In category $i$, the action $a = i$ has features distributed as $\phi(X, a) \sim \mathcal N(0, \Sigma_1)$ where $\Sigma_1 = \diag(1, 10^{-9},\ldots, 10^{-9})$. 
There are also two other actions with variance $1$ at unique coordinates on the diagonal for each context. All contexts also share two actions $a = 5, 6$ such that $\phi(X, 5), \phi(X, 6) \sim \mathcal N(0, \diag(10^{-9}, \ldots, 10^{-9}, 5))$ independently. All remaining actions lead to identically zero features. 
The first $d - 1$ coordinates of $\theta$ were chosen randomly from $\{-1, 1\}^{d - 1}$. The last coordinate is always zero.

We considered four algorithms: (1) a random algorithm that always chooses actions uniformly from $\mathcal A$, (2) a largest norm algorithm that always chooses the feature with the largest norm, (3) an algorithm that always chooses action $1$ and (4) the sampler-planner algorithms proposed in this work. The planner is first run on an independent dataset of the same size as the one to be used for training in order to experience contexts from the same distribution that generates the online contexts. All algorithms are then run in the online phase with reward feedback.  The algorithms are then evaluated based on their resulting policy values during the online phase on a held-out test set. All algorithms used $\lambda  = 1$ and the planner used $\alpha = 1$ as these worked well empirically, though we further explore the impact of $\lambda$ in the next section. 

The synthetic dataset is designed in such a way that most actions are uninformative and that it is often most prudent to get information about the first and last coordinates since there is the most variance in those directions. Thus a random algorithm should perform poorly, spending time on irrelevant actions. The largest norm algorithm also fails: while much is learned about the last coordinate of $\theta$, little is learned about the others which are necessary for making good decisions. Finally, the single action algorithm fails to account for the switching modes of the contexts. Figure~\ref{fig::synth} shows the policy values plotted over the number of samples taken and also the action distributions of the algorithms. Each line represents the mean of $20$ trials and the error bars represent the standard error.

\subsection{Learning to rank dataset}

We evaluate the planner-sampler performance on real-world data from the Yahoo! Learning to Rank challenge \citep{chapelle2011yahoo}, which has been used previously \citep{foster2018practical}.

The ranking dataset is structured as follows. Each datapoint (row) is a $700$-dimensional feature vector associated with a ranking relevance score that takes on values in $\{ 0, 1, 2, 3, 4\}$. Each datapoint is also associated with a single context, termed a query. Multiple datapoints matched with the same query correspond to actions (documents) that the learner can choose when presented with that query. The dataset is already divided into training, validation, and testing data.

The objective is to choose actions, among those presented for a given query, that maximize the relevance scores. Unlike the previous synthetic dataset, the reward function here is misspecified, i.e., it does not exactly follow a linear model. 
To reduce the computational burden, we randomly subsampled coordinates of the original $700$-dimensional features down to $300$-dimensional features. We then normalized them to ensure the norms are at most $1$. Similar to prior work, we also limit the number actions (documents) available at any given context to $K = 20$.

\begin{wrapfigure}{r}{0.45\textwidth}
\vspace{-0.75cm}
    \centering
\includegraphics[width=\linewidth]{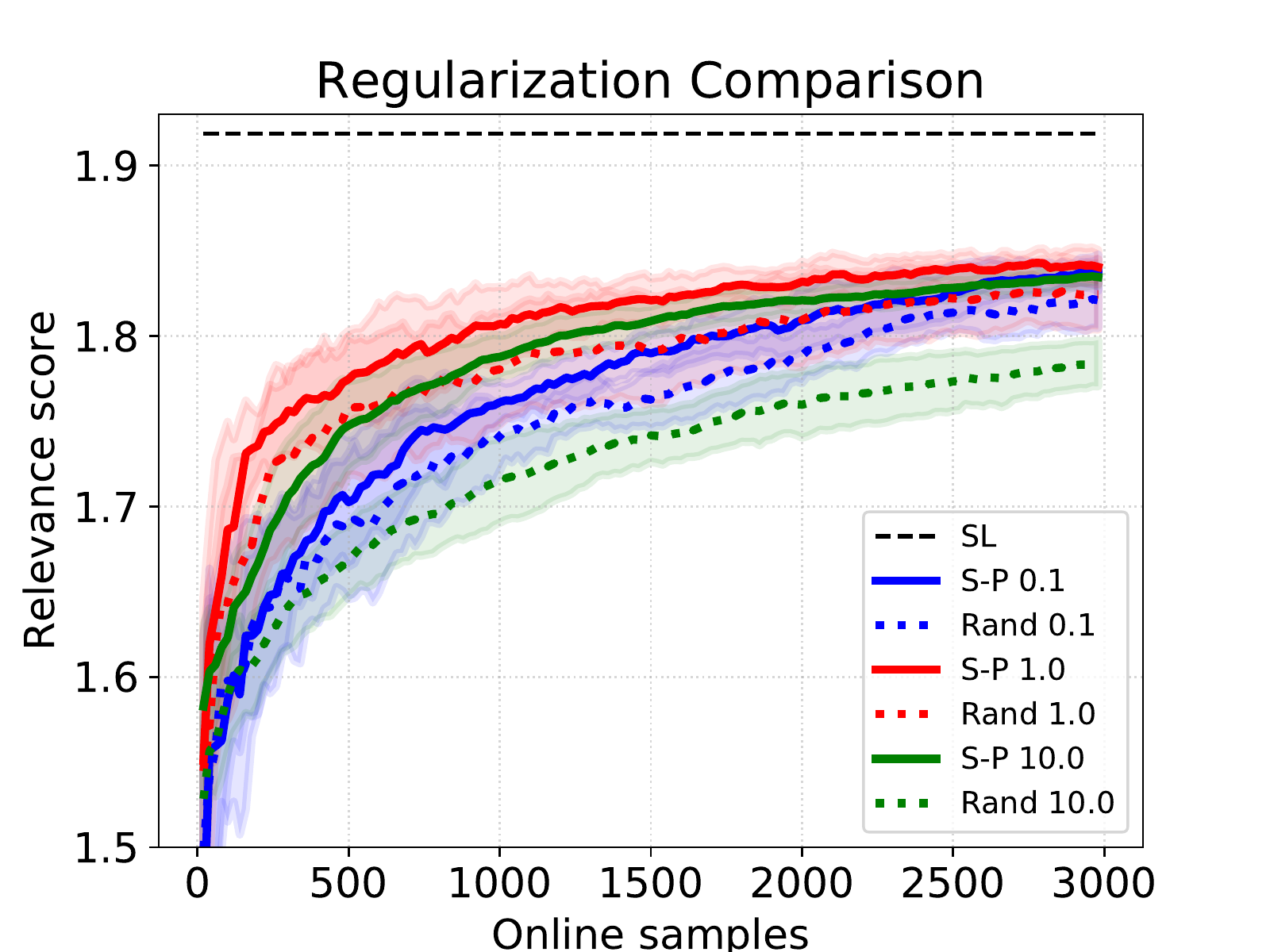} 
  \caption{The sampler-planner (S-P) pair is compared with a random sampler (Rand) with varying levels of regularization $\lambda$. The supervised learning oracle (SL), which observes state-action pairs in the training set, is shown as baseline.}\label{fig::real-reg}
  \vspace{-0.5cm}
\end{wrapfigure}
As before, we compare the planner-sampler pair to  a random algorithm and a supervised learning oracle that observes all the context-action features and relevance scores, which is 200,000 training samples in total. To simulate streaming the data online to the algorithms, we sequentially iterate through the contexts and present the associated features to the learner as actions.
During the offline phase, we run the planner on the validation set where none of the relevance scores are observed in order to generate the sampling policies. 
During the online phase, we run both the sampler and the random algorithm on the training data such that they observe the same contexts but may take different actions. The maximum number of distinct contexts we can iterate over is about $3000$.  The resulting policies are evaluated on the full test set after every 20 training observations. This experiment was repeated with varying values of the regularization parameter $\lambda \in \{ 0.1, 1, 10\}$. The planner uses $\alpha = 1$; we did not find significant improvements with $\alpha < 1$. Additional values for $\lambda$ are reported in the supplementary material.

Figure~\ref{fig::real-reg} depicts the results. Each point on the graph represents the mean of $10$ independent trials and the error bands represent the standard deviation. The randomness is due only to internal randomness in the algorithms and the randomized order in which the contexts are revealed. The sampler-planner consistently outperforms the random algorithm in most regimes, including $\lambda =  1$ where the pinnacle performance of both algorithms is achieved. The gap between the algorithms is widened for larger choices of $\lambda$. We notice that for $\lambda = 0.1$ the curves appear to catch up as more data is gathered; it is possible that they  would outperform larger value of regularization if more samples  were available. The optimal choice for $\lambda$ depends on the bias-variance tradeoff for such problem. Note that despite misspecification, our sampler-planner approach consistently matches or significantly exceeds the performance of the baseline approach, particularly when the number of samples is small.

%% file: contexts/6-conclusion.tex
\section{Conclusion}
Assuming access to a set of offline contexts, this work proposes an algorithm to find a single stochastic policy to acquire data strategically. The procedure is sound as validated by a theoretical analysis and by experiments even when the reward model is not linear. 

Our work opens several interesting directions: (1) can we devise a procedure that works across all values of regularization without severely impacting the amount of offline data we need? (2) how does the performance of the algorithm degrade if the online context distribution differs from the offline one? (3) can we incorporate prior reward data to reduce the sample complexity, and how is the benefit of such data impacted by the size of the  state-action space, and the decision policy used to gather the prior reward data, which might be an adversarial procedure?

Our work presents little direct  negative potential societal impacts though whether contextualized decision policies are beneficial or harmful depends on the setting and the optimization (reward function) chosen by a system designer. 

%% file: contents/9-concentration.tex
\good{AZ: check that the lazy update is handled correctly in appendix at this stage}

\tableofcontents

\section{Proof of Theorem \ref{thm:main}}
\begin{proof}
We start by invoking \fullref{prop:LinRegUnionBound}. Define a shorthand for the RHS of \cref{prop:LinRegUnionBound}
\begin{align}
	\sqrt{\beta}\defeq \min\{ \alpha_1,\alpha_2\} +\sqrt{\lambda_{reg}}\| \theta^\star\|_2
\end{align}
With a rescaling on $\delta$ we write
\begin{align}
	\Pro\(  \underbrace{\forall (s,a): \quad  |\phi(s,a)^\top(\theta^\star - \widehat\theta)| \lesssim \| \phi(s,a)\|_{(\Sigma'_N)^{-1}}\ln(N)\sqrt{\beta}}_{\Eps_1}\) \geq 1-\frac{\delta}{4}.
\end{align}
Notice that under the event $\Eps_1$ above
$$\E_{s \sim \mu}\max_{a \in \ActionSpace_s} |\phi(s,a)^\top(\theta^\star - \widehat\theta)| \leq \ln(N)\sqrt{\beta}\E_{s \sim \mu} \max_{a \in \ActionSpace_s} \| \phi(s,a)\|_{(\Sigma'_N)^{-1}} = \ln(N)\sqrt{\beta}u'_N$$ 
Recall that $K$ is the number of policy switches by the sampler and Lemma~\ref{lem:Switches} gives $K=\tildeO(d)$. We set 
\begin{align}
	\lambda_{reg} & = \Omega(\ln (d/\delta)), \\
	M &= \Omega\(\frac{KN}{\lambda_{reg}} \ln\frac{d^2N}{\lambda_{reg}\delta}\),
	\label{eqn:preconditions}
\end{align}
then \fullref{lem:OfflineOnline} relates the expected online uncertainty $u'_N$ with the expected offline uncertainty $u_M$
\begin{align}
	\Pro\( \underbrace{u'_{N} \lesssim u_{M}}_{\Eps_2}\) \geq 1-\frac{\delta}{2}
\end{align}
Finally, the expected offline uncertainty is bounded by in \fullref{lem:unbound} chained with \fullref{lem:SumOfUncertainties}
\begin{align}
	\Pro\(\underbrace{u_{M} \lesssim \frac{1}{M}\ln \frac{1}{\delta} + \sqrt{\frac{1}{\alpha M}  d\ln\(\frac{d\lambda_{reg}+ M}{d}\)}}_{\Eps_3} \) \geq 1-\frac{\delta}{4}.
\end{align}
A union bound over the events $\Eps_1,\Eps_2,\Eps_3$ and chaining the statements now produces with probability at least $1-\delta$
\begin{align}
	\E_{s \sim \mu}\max_{a \in \ActionSpace_s} |\phi(s,a)^\top(\theta^\star - \widehat\theta)| & \lesssim \frac{\ln(N)\sqrt{\beta}}{M}\ln \frac{1}{\delta} + \ln(N)\sqrt{\frac{\beta}{\alpha M}  d\ln\(\frac{d\lambda_{reg}+ M}{d}\)} \\
	 & \lesssim \frac{\ln(N)\sqrt{\beta}}{M}\ln \frac{1}{\delta} + \ln(N)\sqrt{\frac{\beta}{N}  d\ln\(\frac{d\lambda_{reg}+ M}{d}\)} \label{eqn:boundphi}
\end{align}
after recalling $\alpha = N/M$. As a result, we set
$$
N = \Omega\( \frac{d\beta}{\epsilon^2} \ln^3 \pbra{\lambda_{reg}+\frac{d^2\beta}{\lambda_{reg}\epsilon\delta}} \)
$$
and
$$
M = \Omega \(\frac{K d \beta}{ \lambda_{reg}\epsilon^2}\ln^4 \pbra{\lambda_{reg}+\frac{d^2\beta}{\lambda_{reg}\epsilon \delta}} \).
$$
As a result, we satisfy the preconditions in Eq.~\eqref{eqn:preconditions} because 
\begin{align*}
	&\frac{KN}{\lambda_{reg}} \ln\frac{d^2N}{\lambda_{reg}\delta} \\
	=\;& \frac{K \frac{d\beta}{\epsilon^2} \ln^3 \pbra{\lambda_{reg}+\frac{d^2\beta}{\lambda_{reg}\epsilon\delta}}}{\lambda_{reg}} \ln \(\frac{d^2}{\lambda_{reg}\delta} \frac{d\beta}{\epsilon^2} \ln^3 \pbra{\lambda_{reg}+\frac{d^2\beta}{\lambda_{reg}\epsilon\delta}} \) \\
	\lesssim \;& \frac{K d \beta}{ \lambda_{reg}\epsilon^2}\ln^3 \pbra{\lambda_{reg}+\frac{d^2\beta}{\lambda_{reg}\epsilon \delta}}3\ln \(\frac{d}{\lambda_{reg}\delta} \frac{d\beta}{\epsilon} \ln \pbra{\lambda_{reg}+\frac{d^2\beta}{\lambda_{reg}\epsilon\delta}} \)\tag{Note that $\frac{d\beta}{\lambda_{reg}\delta}\ge 1$}\\
	\lesssim \;& \frac{K d \beta}{ \lambda_{reg}\epsilon^2}\ln^4 \pbra{\lambda_{reg}+\frac{d^2\beta}{\lambda_{reg}\epsilon \delta}}.
\end{align*}

Note that $\epsilon\le Kd \sqrt{\beta}/\lambda_{reg}$. Then 
\begin{align}
\frac{ \ln(N)\sqrt{\beta}}{M}\ln \frac{1}{\delta} & \lesssim    \frac{\epsilon \ln(N)}{\ln^4 \pbra{\lambda_{reg}+\frac{d^2\beta}{\lambda_{reg}\epsilon \delta}} }\ln \frac{1}{\delta}\label{equ:pf-thm1-6}
% & = 
% \frac{\epsilon \ln \( 
%\frac{d\beta}{\epsilon^2} \ln^3 \pbra{\lambda_{reg}+\frac{d^2\beta}{\lambda_{reg}\epsilon\delta}}   \) }{\ln^4 \pbra{\lambda_{reg}+\frac{d^2\beta}{\lambda_{reg}\epsilon \delta}} }\ln \frac{1}{\delta}.
\end{align}
By the facts $\ln(x\ln(y))\le \ln(xy)\le 2\ln(\max(x,y))$ and $\ln(x^k)=k\ln(x),$ we have
\begin{align}
&\ln(N)=\ln \(\frac{d\beta}{\epsilon^2} \ln^3 \pbra{\lambda_{reg}+\frac{d^2\beta}{\lambda_{reg}\epsilon\delta}}\)
\le 3\ln \(\frac{d\beta}{\epsilon} \ln \pbra{\lambda_{reg}+\frac{d^2\beta}{\lambda_{reg}\epsilon\delta}}\)\\
\le\; & 6\ln \(\max\(\frac{d\beta}{\epsilon},\lambda_{reg}+\frac{d^2\beta}{\lambda_{reg}\epsilon\delta}\)\)\le 6\ln \(\lambda_{reg}+\frac{d^2\beta}{\lambda_{reg}\epsilon\delta}\).\label{equ:pf-thm1-lnn}, 
\end{align}
since $\lambda$ is at most $d$ so $\frac{d}{\lambda \delta} \geq 1.$
Continuing Eq.~\eqref{equ:pf-thm1-6}, we get
\begin{align}
	\frac{ \ln(N)\sqrt{\beta}}{M}\ln \frac{1}{\delta}\lesssim    \frac{\epsilon \ln(N)}{\ln^4 \pbra{\lambda_{reg}+\frac{d^2\beta}{\lambda_{reg}\epsilon \delta}} }\ln \frac{1}{\delta}\lesssim \epsilon.
\end{align}
%$\le \epsilon/2$. 
Note also that under this choice for $N$  the right most term of Eq.~\eqref{eqn:boundphi} can be upper bounded by 
\begin{align}
&\ln(N)\sqrt{\frac{\beta}{N}  d\ln\(\frac{d\lambda_{reg}+ M}{d}\) } \lesssim \epsilon \ln(N)  \sqrt{  \ln^{-3} \pbra{\lambda_{reg}+\frac{d^2\beta}{\lambda_{reg}\epsilon\delta}} \ln \( \frac{d\lambda_{reg}+ M}{d}\) }.\label{equ:pf-thm1-7}
%\lesssim \;& \epsilon  \ln \(\lambda_{reg}+\frac{Kd^2\beta}{\lambda_{reg}\epsilon\delta} \)\sqrt{\ln^{-2} \(\lambda_{reg}+\frac{Kd^2\beta}{\lambda_{reg}\epsilon\delta} \)}\\
%\lesssim \;& \epsilon.
\end{align}
Recall that $K=\tildeO(d)\le O(d^2).$ The logarithmic term can be upper bounded by
\begin{align}
	&\ln \( \frac{d\lambda_{reg}+ M}{d}\)=\ln\(\lambda_{reg}+\frac{K\beta}{\lambda_{reg}\epsilon^2}\ln^4 \pbra{\lambda_{reg}+\frac{d^2\beta}{\lambda_{reg}\epsilon \delta}}\)\\
	\lesssim\;& \ln\(\lambda_{reg}+\frac{d^2\beta}{\lambda_{reg}\epsilon^2}\ln^4 \pbra{\lambda_{reg}+\frac{d^2\beta}{\lambda_{reg}\epsilon \delta}}\)\\
	\le \;& 4\ln\(\lambda_{reg}+\frac{d^2\beta}{\lambda_{reg}\epsilon}\ln \pbra{\lambda_{reg}+\frac{d^2\beta}{\lambda_{reg}\epsilon \delta}}\)\\
	\le \;& 8\ln\(\lambda_{reg}+\frac{d^2\beta}{\lambda_{reg}\epsilon\delta}\).
\end{align}
Now we continue with Eq.~\eqref{equ:pf-thm1-7}
\begin{align*}
	& \epsilon \ln(N) \sqrt{  \ln^{-3} \pbra{\lambda_{reg}+\frac{d^2\beta}{\lambda_{reg}\epsilon\delta}} \ln \( \frac{d\lambda_{reg}+ M}{d}\) }\\
	\lesssim \;& \epsilon \ln(N) \sqrt{  \ln^{-2} \pbra{\lambda_{reg}+\frac{d^2\beta}{\lambda_{reg}\epsilon\delta}}}\\
	\lesssim \;& \epsilon \ln \(\lambda_{reg}+\frac{d^2\beta}{\lambda_{reg}\epsilon\delta}\)\sqrt{  \ln^{-2} \pbra{\lambda_{reg}+\frac{d^2\beta}{\lambda_{reg}\epsilon\delta}}} \tag{By Eq.~\eqref{equ:pf-thm1-lnn}}
	\lesssim \epsilon.
\end{align*}

Rescaling $\epsilon$ by a constant, we ensure
\begin{align}\label{equ:pf-thm1-0}
	\Pro\( \E_{s \sim \mu}\max_{a \in \ActionSpace_s} |\phi(s,a)^\top(\theta^\star - \widehat\theta)| \leq \epsilon \) \geq 1-\delta.
\end{align}

Now we prove Eq.~\eqref{equ:thm1-2}. Let $\pi^\star(s)\defeq \argmax_{a\in\calA_s}\phi(s,a)^\top \thetas.$ Define  $\Delta(s)\defeq \max_{a \in \ActionSpace_s} |\phi(s,a)^\top(\theta^\star - \widehat\theta)|$ for shorthand. It follows immediately that 
\begin{align}
	{\phi(s,\widehat\pi(a))}^\top{\thetas}&\ge {\phi(s,\widehat\pi(a))}^\top{\widehat\theta}-\Delta(s),\label{equ:pf-thm1-1}\\
	{\phi(s,\pi^\star(a))}^\top{\thetas}&\le {\phi(s,\pi^\star(a))}^\top{\widehat\theta}+\Delta(s).\label{equ:pf-thm1-2}
\end{align}
By the definition of $\widehat\pi$, we have
\begin{align}
	{\phi(s,\widehat\pi(a))}^\top{\widehat\theta}\ge{\phi(s,\pi^\star(a))}^\top{\widehat\theta}.\label{equ:pf-thm1-3}
\end{align}
Combining Eqs.~\eqref{equ:pf-thm1-1},~\eqref{equ:pf-thm1-2}, and~\eqref{equ:pf-thm1-3} we get
\begin{align}
	{\phi(s,\widehat\pi(a))}^\top{\thetas}\ge{\phi(s,\pi^\star(a))}^\top{\thetas}-2\Delta(s).
\end{align}
Consequently, 
\begin{align}\label{equ:pf-thm1-4}
	\E_{s\sim \mu}\bbra{\pbra{\phi(s,\pi^\star(a))-\phi(s,\pi(a))}^\top{\thetas}}\le 2\E_{s\sim \mu}\bbra{\Delta(s)}.
\end{align}
As a result, under the same event indicated by Eq.~\eqref{equ:pf-thm1-0}, we get
\begin{align}\label{equ:pf-thm1-5}
	\E_{s \sim \mu}\bbra{ \( \phi(s,\pi^\star(s))-\phi(s,\pihat(s)) \)^\top\theta^\star} \leq 2\epsilon,
\end{align}
which is exactly Eq.~\eqref{equ:thm1-2}.
\end{proof}

\section{Linear Regression}
\begin{proposition}[Linear Regression Non-Adaptive Setting]
\label{prop:LinReg}
Consider drawing $n$ i.i.d. copies of $\phi_i$ from some fixed distribution, and define
	$$\widehat \theta = \( \sum_{i=1}^n \phi_i\phi_i^\top +\lambda_{reg} I\)^{-1} \sum_{i=1}^n \phi_i y_i	$$
	where 
	\begin{align}
		y_i = \phi_i^\top \theta^\star + \eta_i
	\end{align}
	for some fixed $\theta^\star$ 
	and $\eta_i$ is mean zero $1$-subgaussian conditioned on $\phi_i$.
	Then for any fixed vector $x$
	$$
	\Pro\( x^\top (\theta^\star - \widehat \theta )  \leq \| x\|_{\Sigma^{-1}}\( \sqrt{2\ln\frac{2}{\delta}} + \sqrt{\lambda_{reg}}\|\theta^\star\|_{2}\)\) \geq 1-\delta.
	$$
\end{proposition}
% %|\sum_{i=1}^nx^\top \Sigma^{-1}\phi_i\eta_i| \leq \| x\|_{\Sigma^{-1}}\(
\begin{proof}
\begin{align}
	x^\top (\theta^\star - \widehat \theta ) & = x^\top \theta^\star - x^\top \( \sum_{i=1}^n \phi_i\phi_i^\top +\lambda_{reg} I\)^{-1} \( \sum_{i=1}^n \phi_i \(\phi_i^\top \theta^\star + \eta_n \) + \lambda_{reg} \theta^\star  -\lambda_{reg} \theta^\star  \) \\
	& = - x^\top \Sigma^{-1}\sum_{i=1}^n\phi_i\eta_i + \lambda_{reg} x^\top \Sigma^{-1}\theta^\star.
\end{align}	
Using Cauchy-Schwartz we have
\begin{align}
	\lambda_{reg} x^\top \Sigma^{-1}\theta^\star \leq  \lambda_{reg} \| x\|_{\Sigma^{-1}}\|\theta^\star\|_{\Sigma^{-1}}.
\end{align}
We have that $\lambda_{reg} \|\theta^\star\|_{\Sigma^{-1}} \leq \sqrt{\lambda_{reg}}\|\theta^\star\|_{2}$. Since the feature vectors $\phi_i$'s are sampled from a fixed distribution, conditioned on the sampled state-actions $\phi_i$, both the covariance matrix $\Sigma$ is fixed and the noise $\eta_i$ is independent and $1$ sub-Gaussian. Define $v_i = x^T \Sigma^{-1} \phi_i \eta_i$. Then conditioned on the sampled state-actions $\phi_i$, the $v$s are independent random $(x^T \Sigma^{-1} \phi_i)^2$ sub-Gaussian random variables and we can apply Hoeffding's inequality, conditioned on the observed state-action features $\phi_i$:
%Using Hoeffding inequality for independent scaled ($x^T \Sigma^{-1} \phi_i) sub-Gaussian random variables we get
\begin{align}
	\Pro\(|\sum_{i=1}^nx^\top \Sigma^{-1}\phi_i\eta_i| \leq \sqrt{2\sum_{i=1}^n (x^\top \Sigma^{-1}\phi_i)^2 \ln\frac{2}{\delta}} = \| x\|_{\Sigma^{-1}}\sqrt{2\ln\frac{2}{\delta}} \) \geq 1-\delta.
\end{align}
Combining with the regularization part, we conclude.
\end{proof}

\begin{proposition}[Linear Regression with Union Bound]
\label{prop:LinRegUnionBound}
In the same setting as \cref{prop:LinReg} we have that
	$$
	\Pro\(\forall (s,a): 
	\phi(s,a)^\top (\theta^\star - \widehat \theta )  \leq  \| \phi(s,a) \|_{\Sigma^{-1}}\( \min\{ \alpha_1,\alpha_2\} +\sqrt{\lambda_{reg}}\| \theta^\star\|_2 \) \) \geq 1-\delta.
	$$
	where
	\begin{align}
	\alpha_1 & \defeq \alphaone\\
	 \alpha_2 & \defeq \alphatwo
	\end{align}
\end{proposition}
% 	\quad |\sum_{i=1}^n\phi(s,a)^\top \Sigma^{-1}\phi_i\eta_i|
\begin{proof}
	The proof essentially follows from taking the sharper of two results obtained as follows:
	\begin{enumerate}
		\item Invoke \cref{prop:LinReg} with a union bound over the state-action space (this gives the bound with the $\alpha_1$ term)
		\item This follows from a discretization argument; see e.g., \citep{lattimore2020bandit}, in particular, their equation 20.3.
	\end{enumerate}
\end{proof}

\section{Bounding the Uncertainty}
\underline{We use the notation $u_m = \E_m U_m$ and $u'_n = \E'_n U'_n$ for short.}
\begin{lemma}[Relations between  Offline and Online Uncertainty]
\label{lem:OfflineOnline}
If $\lambda_{reg} = \Omega(\ln\frac{d}{\delta})$ and $
	M = \Omega\(\frac{KN}{\lambda_{reg}} \ln\frac{dNK}{\lambda_{reg}\delta}\)
$, upon termination of  \cref{alg:Planner,alg:Sampler} it holds that
\begin{align*}
\Pro\( u'_{N} \lesssim u_{M}\) \geq 1-\frac{\delta}{2}
\end{align*}
where  the probability is over the offline context dataset $\C = \{s_1,\dots,s_M\}$ and the online  context dataset $\C' = \{s'_1,\dots,s'_N\}$.
\end{lemma}
\begin{proof}

Let $d_1,\dots,d_M$ be the conditional distributions of the feature vectors sampled at timesteps $1,\dots,n$ in \cref{alg:Planner} after the algorithm has terminated.
Conditioned on $\mathcal F_M = \sigma(s_1,\dots,s_M)$, the $d_i$'s are non-random; let \good{make sure the def below is consistent everywhere}
\begin{align}
\overline \Sigma & = \alpha \sum_{i=1}^M \E_{\phi \sim d_i} \phi\phi^\top + \lambda_{reg} I
\end{align}
be the conditional expectation of the cumulative covariance matrix and let 
\begin{align}
\Sigma'_n & = \sum_{i=1}^n  (\phi'_i)(\phi'_i)^\top + \lambda_{reg} I
\end{align}
be the cumulative covariance matrix experienced in \cref{alg:Sampler}
where $\phi'_i \defeq \phi(s'_n,a'_n)$ is the sampled feature during the execution of \cref{alg:Planner}.
If $\lambda_{reg} = \Omega(\ln\frac{d}{\delta})$ and $
	M = \Omega\(\frac{KN}{\lambda_{reg}} \ln\frac{dNK}{\lambda_{reg}\delta}\)
$
 we obtain from \fullref{lem:SigmaOffline} and \fullref{lem:SigmaOnline}
\begin{align}
\Pro\(\forall x, \| x \|_2\leq 1:   \| x \|_{\overline \Sigma^{-1}} \le 2\;   \| x \|_{\Sigma^{-1}_M}  \) \geq 1-\frac{\delta}{4} \\
\Pro\(\forall x, \| x \|_2\leq 1:   \| x \|_{(\Sigma_N')^{-1}} \le 9\;   \| x \|_{\overline \Sigma^{-1}}  \) \geq 1-\frac{\delta}{4}.
\end{align}
Define the policy maximizing the online uncertainty 
\begin{align}
\pi'_{n}(s) & = \argmax_{a \in \ActionSpace_s} \| \phi(s,a) \|_{(\Sigma'_n)^{-1}}.
\end{align}
Under the two above events 
we can write 
\begin{align}
u'_N & \defeq \E_{s \sim \mu} \max_{a \in \ActionSpace_s} \| \phi(s,a) \|_{(\Sigma'_N)^{-1}} \\
& = \E_{s \sim \mu}\| \phi(s,\pi'_N(s)) \|_{(\Sigma'_N)^{-1}}  \\
& \leq 9\E_{s \sim \mu}\| \phi(s,\pi'_N(s)) \|_{\overline \Sigma^{-1}} \\
& \leq 18\E_{s \sim \mu}\| \phi(s,\pi'_N(s)) \|_{\Sigma_M^{-1}} \\
& \leq 18\E_{s \sim \mu} \max_{a} \| \phi(s,a) \|_{\Sigma_{M}^{-1}} \\
& \leq  18\E_{s \sim \mu}\max_{a} \| \phi(s,a) \|_{\Sigma_{\underline M}^{-1}} \\
& = 18u_M.
\end{align}
\end{proof}

\begin{lemma}[Offline Expected Uncertainty] 
\label{lem:unbound}
We have that
$$
\Pro\(u_{M} \leq \frac{1}{M}\sum_{m=1}^M u_m  \lesssim \frac{1}{M} \Big[\ln\frac{1}{\delta} + \sum_{m=1}^M U_m \Big]  \defeq  \frac{\mathcal R(M,\frac{\delta}{4})}{M} \) \geq 1-\frac{\delta}{4}.
$$
\end{lemma}
\begin{proof}
Consider the event that the sum of the predictable means $\sum_{m=1}^{M} u_m = \sum_{m=1}^{M} \E [U_m \mid \F_m] $ does not deviate significantly from $\sum_{m=1}^{M} U_m$:
\begin{align}
\mathcal E(\delta) \defeq \Bigg\{\sum_{n=1}^M u_m \leq \frac{1}{4}\( c_1(\delta) + \sqrt{c_1^2(\delta)+4\(\sum_{m=1}^M U_m +c_2(\delta)\)}\)^2  \defeq \mathcal Rhs(M,\delta) \Bigg\}.
\end{align}

Using \fullref{thm:ReverseBernstein} (which also defines $c_1,c_2$) we obtain
\begin{align}
\Pro\( \mathcal E(\delta/4) \) \geq 1-\frac{\delta}{4}.
\end{align}
%Define $D \defeq \sum_{m=1}^{M} U_m  =  \sum_{m=1}^{M} \max_{a \in \ActionSpace_{S_m}} \| \phi(S_m,A_m) \|_{\Sigma_{\underline m}^{-1}}$ and 
From \fullref{lem:DecreasingUncertainty} we know that the sequence $\{u_n\}_{n=1}^{N+1}$ is surely decreasing, which means that the last element must be less than the average:
\begin{align}
u_{M} \leq \frac{1}{M} \sum_{n=1}^M u_m = \frac{\mathcal Rhs(M,\delta/4)}{M} 
\end{align}
where in particular the equality holds under the event $\mathcal E(\delta)$. Finally using Cauchy-Schwartz we conclude that under the same event
\begin{align}
\sum_{n=1}^M u_m \lesssim c^2_1(\delta) + c_2(\delta) + \sum_{m=1}^M U_m \defeq \mathcal R(M,\delta).\end{align}
\end{proof}

\begin{lemma}[Decreasing Uncertainty] 
\label{lem:DecreasingUncertainty}
For every $n$ it holds that	
$$
u_{n+1} \leq u_{n}.
$$
\end{lemma}
\begin{proof}
By linear algebra, we must have 
$$\Sigma_{\underline{n+1}} \succeq \Sigma_{\underline n} \qquad \longrightarrow \quad \Sigma_{\underline{n+1}}^{-1} \preceq \Sigma_{\underline n}^{-1}$$ 
Using the definitions we have:
\begin{align}
u_{n+1} & = \E_{s \sim \mu} \max_{a \sim \ActionSpace_s}\|\phi(s,a) \|_{\Sigma_{\underline {n+1}}^{-1}} \\ 
& \leq \E_{s \sim \mu} \max_{a \sim \ActionSpace_s}\|\phi(s,a) \|_{\Sigma_{\underline n}^{-1}} \\
& = u_{n}.
\end{align}
\end{proof}

\begin{lemma}[Sum of Observed Uncertainties]
\label{lem:SumOfUncertainties}
\good{AZ: This lemma is tricky with all the scaling; make sure we get this right cause extra factors might pop up from this}
If $\lambda_{reg} \geq 1$ and and $\alpha \leq 1$ then 
\begin{align}
	D & \defeq \sum_{n=1}^{M} U_m \lesssim \sqrt{\frac{M}{\alpha}  d\ln\(\frac{d \lambda_{reg}+ M}{d}\)}. 
\end{align} 
\end{lemma}
\begin{proof}
Let $\{\phi_m\}$ be the feature vectors experienced during the offline phase; we have
\begin{align}
\sum_{m=1}^{M} U_m & = \sum_{m=1}^{M} \| \phi_m \|_{\Sigma_{\underline m}^{-1}} \\
	& = \sum_{m=1}^{M} \| \phi_m \|_{\(\alpha\sum_{j=1}^{\underline m} \phi_j\phi_j + \lambda_{reg} I\)^{-1}} \\
	& \leq \sqrt{M \sum_{m=1}^{M} \| \phi_m \|^2_{\(\alpha\sum_{j=1}^{\underline m} \phi_j\phi_j + \lambda_{reg} I\)^{-1}}} \\
		& = \sqrt{\frac{M}{\alpha} \sum_{m=1}^{M} \| \sqrt{\alpha}\phi_m \|^2_{\(\sum_{j=1}^{\underline m} \sqrt{\alpha} \phi_j\sqrt{\alpha} \phi_j + \lambda_{reg} I\)^{-1}}} \\
	& \leq  \sqrt{\frac{M}{\alpha} \times 3\times \underbrace{ \ln \frac{\det \(  \sum_{j=1}^{M} \sqrt{\alpha}\phi_j\sqrt{\alpha}\phi_j^\top + \lambda_{reg} I  \)}{\det \(\lambda_{reg} I  \)}}_{\defeq \mathcal I}}
	 \end{align}
where the first inequality follows from Cauchy-Schwartz. The  second inequality follows from \fullref{lem:potential}, where the precondition for \fullref{lem:potential} is satisfied by \fullref{lem:max_det_ratio} since $\alpha \leq 1$.
%e.g., Lemma 29 (Potential Argument) in \cite{zanette2021cautiously} with the precondition being satisfied by their Lemma 27 (Maximum Determinant Ratio)\footnote{Note that our $m$ and $\underline{m}$ corresponds to those authors' $n$ and $\underline{n}$ respectively.} \good{AZ: should make this reference self contained but it's not a priority now}.
Finally, to bound the information gain $\mathcal I$, note $||\sqrt{\alpha} \phi_i||_2 \leq 1$ since $\alpha \leq 1$. Then  \citep{abbasi2011improved}'s Lemma 11  
\begin{align*}
	   &\ln \det \(  \sum_{j=1}^M \sqrt{\alpha} \phi_j \sqrt{\alpha} \phi_j^T + \lambda_{reg} I \)
	   - \ln \det \(  \lambda_{reg} I) \)  \\
	 \leq\;&  d \ln ((\Tr( \lambda_{reg} I) + M) / d) - \ln \det (\lambda_{reg} I)
\end{align*}
Since 
\begin{align}
	 \ln \det \( \lambda_{reg} I  \) & = d \ln  \( \lambda_{reg}  \) \geq 1.
\end{align}
Then 
\begin{align}
	 \ln \det \(  \sum_{j=1}^{M} \sqrt{\alpha}\phi_j\sqrt{\alpha}\phi^\top_j + \lambda_{reg} I  \) & \leq d\ln\(\frac{d \lambda_{reg} + M}{d}\)
\end{align}
\end{proof}

\begin{lemma}[Elliptical Potential Argument Lemma with Doubling](see \cite[Lemma 36]{zanette2021cautiously})
\label{lem:potential}
	Let $x_1,\cdots,x_M$ be a sequence of vectors such that $\normtwo{x_i}\le 1.$ Define $\Sigma_m=\lambda_{reg}I+\sum_{i=1}^{m-1}x_ix_i^\top.$
	Suppose $\underline{m}\le m$ satisfies $\det(\Sigma_m)\le 4\det(\Sigma_{\underline m}).$ Then we have
	\begin{align}
		\sum_{m=1}^{M}\norm{x_i}{\Sigma_{\underline m}^{-1}}\le 3\ln\frac{\det(\Sigma_{M+1})}{\det(\lambda_{reg} I)}.
	\end{align}
\end{lemma}

\begin{lemma}[Maximum Determinant Ratio](see \cite[Lemma 34]{zanette2021cautiously})
\label{lem:max_det_ratio} 
	Let $x_1,\cdots,x_M$ be a sequence of vectors such that $\normtwo{x_i}\le 1.$ and assume $\lambda \geq 1$.  Define $\Sigma_m=\lambda_{reg}I+\sum_{i=1}^{m-1}x_ix_i^\top.$
	Then for $\underline{m}\le m$ we have $\det(\Sigma_m)\le 4\det(\Sigma_{\underline m}).$
\end{lemma}

\section{Matrix Concentration Inequalities}
In this section we present matrix concentration inequalities used in our proof.

\subsection{Known Matrix Concentration Inequalities}

The following result about eigenvalues lower and upper bounds is well known.
\begin{lemma}[Theorem 1.1 of \citet{tropp2012user}]\label{lem:matrix-chernoff}
	Let $X_1,\cdots,X_n$ be a sequence of independent, positive semi-definite, self-adjoint matrices with dimension $d$. Suppose $\lmax(X_k)\le R$ almost surely. Define $\mu_{\rm min}=\lmin(\sum_{k}\E X_k)$ and $\mu_{\rm max}=\lmax(\sum_{k}\E X_k).$ Then
	\begin{align}
	&\Pr\pbra{\lmin\pbra{\sum_{k=1}^{n} X_k}\le (1-\delta)\mu_{\rm min}}\le d\pbra{\frac{e^{-\delta}}{(1-\delta)^{1-\delta}}}^{\mu_{\rm min}/R}\text{ for }\delta\in[0,1),\\
		&\Pr\pbra{\lmax\pbra{\sum_{k=1}^{n} X_k}\ge (1+\delta)\mu_{\rm max}}\le d\pbra{\frac{e^{\delta}}{(1+\delta)^{1+\delta}}}^{\mu_{\rm max}/R}\text{ for }\delta \geq 0.
	\end{align}
\end{lemma}
We loosen the above inequalities to make them more amenable to a direct use.
\begin{corollary}\label{cor:matrix-chernoff}
	In the setting of Lemma~\ref{lem:matrix-chernoff}, with $\delta\in[0,1]$
	\begin{align}
		&\Pr\pbra{\lmin\pbra{\sum_{k=1}^{n} X_k}\le (1-\delta)\mu_{\rm min}}\le d\pbra{1-\frac{\delta^2}{2}}^{\mu_{\rm min}/R},\\
		&\Pr\pbra{\lmax\pbra{\sum_{k=1}^{n} X_k}\ge (1+\delta)\mu_{\rm max}}\le 	d\pbra{1-\frac{\delta^2}{4}}^{\mu_{\rm max}/R}.
	\end{align}
	In addition, for any $\mu\ge \mu_{\rm max},$
	\begin{align}\label{equ:cor1-3}
		&\Pr\pbra{\lmax\pbra{\sum_{k=1}^{n} X_k}\ge 2\mu}\le 
		d\exp\pbra{-\mu/(4R)}.
	\end{align}
\end{corollary}
\begin{proof}
	The first two inequalities follows from the fact that $\forall \delta \in [0,1]$
	\begin{align}
		\frac{e^\delta}{(1+\delta)^{1+\delta}} & \leq 1-\frac{\delta^2}{4} \\
		\frac{e^{-\delta}}{(1-\delta)^{1-\delta}} & \leq 1-\frac{\delta^2}{2}.
	\end{align}
	Now we prove Eq.~\eqref{equ:cor1-3}. Let $r=\mu/\mu_{\rm max}$. Since $r\ge 1$ we have
	\begin{align}
		&\Pr\pbra{\lmax\pbra{\sum_{k=1}^{n} X_k}\ge (1+\delta)\mu}= \Pr\pbra{\lmax\pbra{\sum_{k=1}^{n} X_k}\ge (r+r\delta)\mu_{\rm max}}\\
		\le\;&\Pr\pbra{\lmax\pbra{\sum_{k=1}^{n} X_k}\ge (1+r\delta)\mu_{\rm max}}\le d\pbra{\frac{e^{r\delta}}{(1+r\delta)^{1+r\delta}}}^{\mu_{\rm max}/R},
	\end{align}
	where the last inequality is due to Lemma~\ref{lem:matrix-chernoff}.

	By basic algebra we have $\frac{e^{x}}{(1+x)^{1+x}}\le e^{-x/4}$ for all $x\ge 1.$ As a result, let $\delta=1$ we have
	\begin{align}
		d\pbra{\frac{e^{r}}{(1+r)^{1+r}}}^{\mu_{\rm max}/R}\le d\exp\pbra{-\frac{r\mu_{\rm max}}{4R}}=d\exp\pbra{-\frac{\mu}{4R}}.
	\end{align}
	Therefore, 
	\begin{align}
		\Pr\pbra{\lmax\pbra{\sum_{k=1}^{n} X_k}\ge 2\mu}\le d\exp\pbra{-\frac{\mu}{4R}}.
	\end{align}
	which completes the proof.
\end{proof}

\subsection{Matrix Concentration Inequalities in All Directions}
In the following development we need to `sandwich' the cumulative matrix around its expectation; as a step towards this, we first derive concentration inequalities as a function of the minimum eigenvalue. 

For the following lemma, see also Lemma 20 \citep{ruan2020linear}. 

\begin{lemma}[Matrix Upper and Lower Bound with Minimum Eigenvalue]\label{lem:matrix-chernoff-reg}
	Let $X_1,\cdots,X_n\sim \calD$ be i.i.d. samples from $\calD$ where $X_k\in \R^{d\times d}.$ Suppose $X_k$ is positive semi-definite for all $k\in [n]$ and $\lmax(X_k)\le 1$ almost surely. Let $\lambda=\lmin(\E_{X\sim \calD}X)>0.$ Then for $\delta\in[0,1]$ we have
	\begin{align}
		&\Pr\pbra{\frac{1}{n}\sum_{k=1}^{n}X_k \preccurlyeq (1+\delta)\E[X]}\ge 1-d\( 1-\frac{\delta^2}{4}\)^{n\lambda} \\
		&\Pr\pbra{\frac{1}{n}\sum_{k=1}^{n}X_k \succcurlyeq (1-\delta)\E[X]}\ge 1-d\( 1-\frac{\delta^2}{2}\)^{n\lambda}.
	\end{align}
\end{lemma}
\begin{proof}
We prove the first inequality first.
	Let $\Sigma=\E[X]$ and define $Y=\Sigma^{-1/2}X\Sigma^{-1/2}$. Then using linearity of expectation we have 
	\begin{align}
		\label{eqn:expecti}
		\E[Y_k] = \E[\Sigma^{-1/2}X_k\Sigma^{-1/2}] =\Sigma^{-1/2}\( \E X_k\)\Sigma^{-1/2}  = \Sigma^{-1/2}\Sigma\Sigma^{-1/2}=I.
	\end{align}
  In addition, 
	$$\normop{Y_k}\le \normop{\Sigma^{-1/2}}\normop{X_k}\normop{\Sigma^{-1/2}}\le 1/\lambda$$ almost surely. As a result, $\lmax(Y_k^\top)\le 1/\lambda$. Now \cref{cor:matrix-chernoff} gives (with $\mu_{max} = \mu_{min} = n$)
\begin{align}
\Pr\pbra{\lmax\pbra{\sum_{k=1}^{n} Y_k} \leq (1-\delta)n} & \leq d\(1-\frac{\delta^2}{2}\)^{n\lambda},\label{equ:reg-1}  \\
\Pr\pbra{\lmax\pbra{\sum_{k=1}^{n} Y_k} \geq (1+\delta)n} & \leq d\(1-\frac{\delta^2}{4}\)^{n\lambda}.\label{equ:reg-2}
\end{align}
	Now, to derive the result for $\delta \in [0,1]$ consider %$(a)$ pre and post-multiplying each side by inside the probability by $\Sigma^{-1}$ and $(b)$  multyplying by $n$ each side and 
	\begin{align*}
		&\Pr\pbra{\frac{1}{n}\sum_{k=1}^{n} X_k\preccurlyeq (1+\delta)\E[X]} = 
		\Pr\pbra{\frac{1}{n}\sum_{k=1}^{n} Y_k\preccurlyeq (1+\delta)I} \tag{By Eq.\eqref{eqn:expecti}}
		\\ =& \Pr\pbra{\lmax\pbra{\sum_{k=1}^{n} Y_k}\le (1+\delta)n} \\
		=& 1- \Pr\pbra{\lmax\pbra{\sum_{k=1}^{n} Y_k}> (1+\delta)n} \\
		\geq& 1-d\( 1-\frac{\delta^2}{4}\)^{n\lambda}.\tag{By Eq.~\eqref{equ:reg-1}}
	\end{align*}
	Finally, to derive the other statement we proceed similarly.
	We have
	\begin{align*}
		&\Pr\pbra{\frac{1}{n}\sum_{k=1}^{n} X_k\succcurlyeq (1-\delta)\E[X]} = 
		\Pr\pbra{\frac{1}{n}\sum_{k=1}^{n} Y_k\succcurlyeq (1-\delta)I} \tag{By Eq.\eqref{eqn:expecti}}
		\\  = &\Pr\pbra{\lmax\pbra{\sum_{k=1}^{n} Y_k}\ge (1-\delta)n} \\
		  =& 1- \Pr\pbra{\lmax\pbra{\sum_{k=1}^{n} Y_k} < (1-\delta)n} \\
		 \geq& 1-d\(1-\frac{\delta^2}{2}\)^{n\lambda}.\tag{By Eq.~\eqref{equ:reg-2}}
	\end{align*}
\end{proof}

Using the lemma just derived, we can derive a matrix upper bound (in all directions) that does not depend on the minimum eigenvalue.
\begin{lemma}[Matrix Upper Bound]
\label{lem:MatrixUpperBound}
	Let $X_1,\cdots,X_n$ be i.i.d. samples from $\calD$ where $X_k\in \R^{d\times d}$. Suppose $X_k$ is positive semi-definite and $\normop{X_k}\le 1$ for all $k\in [n]$ almost surely. For any fixed $\lambda> 0,$ 
	\begin{align}
		\Pr\( \frac{1}{n}\sum_{k=1}^{n}X_k \mle 6\lambda I+2\E_{X\sim \calD}[X] \) \geq 1-2d\exp\pbra{-\frac{n\lambda}{4}}.
	\end{align}
\end{lemma}

\begin{proof}
	Let $\Sigma=\E_{x\sim \calD}[X].$ Consider the spectrum decomposition of $\Sigma$, denoted by $\Sigma=\sum_{k=1}^{d}\lambda_k v_kv_k^\top$. Without loss of generality, we assume $\lambda_1\ge \lambda_2\ge\cdots\ge\lambda_d$. Let $R=\sup\{k:\lambda_k\ge \lambda\}.$ 
	Define $P_+=\sum_{k=1}^{R}v_kv_k^\top$, $P_-=\sum_{k=R+1}^{d}v_kv_k^\top$ and $Q=\sum_{k=1}^{R}v_ke_k^\top\in\R^{d\times R}$ where $e_k\in\R^{R}$ is the $k$-th basis for $R$-dimensional space.
	
	Since $\{v_k\}$ is a set of orthogonal basis, we have $P_++P_-=I.$ By algebraic manipulation we get
	\begin{align}
		&\frac{1}{n}\sum_{k=1}^{n}X_k=(P_++P_-)^\top\pbra{\frac{1}{n}\sum_{k=1}^{n}X_k}(P_++P_-)\\
		=&\frac{1}{n}\sum_{k=1}^{n}P_+^\top X_k P_++\frac{1}{n}\sum_{k=1}^{n}\pbra{P_+^\top X_kP_- + P_-^\top X_kP_+P_-^\top X_kP_-}.\label{equ:lem9-1}
	\end{align}
	Note that for any $u\in \R^{R}$ we have 
	$u^\top Q^\top\Sigma Q u=\sum_{k=1}^{R}\lambda_k\dotp{u}{e_k}^2\ge \lambda \normtwo{u}^2.$
	As a result, $\lmin(Q^\top \Sigma Q)=\lmin(\E_{x\sim \calD}Q^\top X Q)\ge \lambda.$ Consequently, applying Lemma~\ref{lem:matrix-chernoff-reg} we get 	
	\begin{align*}\label{equ:lem8-1}
		&\Pr\( \frac{1}{n}\sum_{k=1}^{n}P_+^\top X_k P_+\mle 2\E_{X\sim \calD}[P_+^\top X P_+] \)\\
		=\;&\Pr\( \frac{1}{n}\sum_{k=1}^{n}QQ^\top X_k QQ^\top\mle 2\E_{X\sim \calD}[QQ^\top X QQ^\top] \)\tag{By the definition of $Q$.}\\
		\ge\;&\Pr\( \frac{1}{n}\sum_{k=1}^{n}Q^\top X_k Q\mle 2\E_{X\sim \calD}[Q^\top X Q] \)\\
		\ge\;& 1-d\( 1-\frac{1}{4}\)^{n\lambda}\ge 1-d\exp\pbra{-\frac{n\lambda}{4}}.
	\end{align*}
	[Note we cannot directly apply Lemma~\ref{lem:matrix-chernoff-reg} to the top of the above sequence because that lemma requires a minimum eigenvalue greater than 0 and $P_+^\top X_k P_+$ is not full rank.] 
	
	Next, by the linearity of expectation we get $\E_{X\sim \calD}[P_+XP_+^\top]=P_+(\E_{X\sim \calD}[X])P_+^\top=P_+\Sigma P_+^\top.$ Recall that the spectrum decomposition gives $\Sigma=\sum_{k=1}^{d}\lambda_kv_kv_k^\top$, where $\dotp{v_k}{v_j}=\ind{j=k}$. As a result,
	\begin{align}
		P_+^\top\Sigma P_+&=\pbra{\sum_{k=1}^{R}v_kv_k^\top}\pbra{\sum_{k=1}^{d}\lambda_k v_kv_k^\top}\pbra{\sum_{k=1}^{R}v_kv_k^\top}\\
		&=\pbra{\sum_{k=1}^{R}\lambda_k v_kv_k^\top}\mle \pbra{\sum_{k=1}^{d}\lambda_k v_kv_k^\top}=\Sigma.
	\end{align}
	Consequently we get $\E_{X\sim \calD}[P_+X^\top P_+]= P_+^\top \Sigma P_+\mle \Sigma=\E_{X\sim \calD}[X].$ Therefore we have
	\begin{align}\label{equ:lem8-3}
		\Pr\( \frac{1}{n}\sum_{k=1}^{n}P_+^\top X_kP_+\mle 2\E_{X\sim \calD}[X] \) \ge 1-d\exp\pbra{-\frac{n\lambda}{4}}.
	\end{align}

	On the other hand, we upper bound the second term in Eq.~\eqref{equ:lem9-1} by Lemma~\ref{lem:matrix-chernoff}.
	Let $Y_k=\pbra{P_+X_kP_-^\top + P_-X_kP_+^\top+P_-X_kP_-^\top}.$ Then we have $\normop{Y_k}\le \normop{Y_k+P_+X_kP_+^\top}\le \normop{X_k}\le 1.$ In addition, 
	\begin{align}
		\E[Y_k]=P_+\Sigma P_-^\top + P_-\Sigma P_+^\top+P_-\Sigma P_-^\top.
	\end{align}
	We claim that $\lmax(\E[Y_k])\le 3\lambda.$ Indeed, for any $v\in S^{d-1}$ we have
	\begin{align}
		v^\top \E[Y_k] v=v^\top P_+\Sigma P_-^\top v+ v^\top P_-\Sigma P_+^\top v+v^\top P_-\Sigma P_-^\top v.
	\end{align}
	Recall that $P_-=\sum_{k=1}^{d}\ind{\lambda_k<\lambda}v_kv_k^\top$ where $\{\lambda_k\},\{v_k\}$ are the eigen-vectors and eigen-values of $\Sigma$. For any $v\in S^{d-1}$, $P_-^\top v$ lies in the linear space spanned by $\{v_i:\lambda_i< \lambda\}$. As a result, $\normtwo{\Sigma P_-^\top v}\le \lambda\normtwo{v}.$ 
	In addition, $\normop{P_+}\le 1$, which leads to $\normtwo{P_+^\top v}\le \normtwo{v}.$ Similarly, $\normtwo{P_-^\top v}\le \normtwo{v}.$
	Consequently, 
	\begin{align}
		v^\top \E[Y_k] v&\le \normtwo{P_+^\top v}\normtwo{\Sigma P_-^\top v}+ \normtwo{\Sigma P_-^\top v} \normtwo{P_+^\top v}+\normtwo{ P_-^\top v}\normtwo{\Sigma P_-^\top v}\\
		&\le 3\lambda\normtwo{v}^2.
	\end{align} for all $v\in S^{d-1}.$ Therefore we prove $\lmax(\E[Y_k])\le 3\lambda.$
	
	Now apply \cref{cor:matrix-chernoff} on $Y_k$ and we get
	\begin{align}\label{equ:lem8-2}
		\Pr\(\frac{1}{n}\sum_{k=1}^{n}Y_k\mle 6\lambda I\)\ge 1-d\exp\pbra{-\frac{3\lambda n}{4}}.
	\end{align}
	Combining Eq.~\eqref{equ:lem8-3} and Eq.~\eqref{equ:lem8-2} we prove the desired result.
\end{proof}

Likewise, we can easily obtain the following matrix lower bound without any dependence on the minimum eigenvalue (see also Lemma 21 from \citep{ruan2020linear}).
\begin{lemma}[Matrix Lower Bound]
\label{lem:MatrixLowerBound}
Let $X_1,\cdots,X_n$ be i.i.d. samples from $\calD$ where $X_k\in \R^{d\times d}$. Suppose $X_k$ is positive semi-definite and rank one and $\normop{X_k}\le 1$ for all $k\in [n]$ almost surely. For any fixed $\lambda \geq 0$ 
	\begin{align}
		\Pr\( 3\lambda I + \frac{1}{n}\sum_{k=1}^{n}X_k \succcurlyeq \frac{1}{8}\E_{X\sim \calD}[X] \) \geq 1-2d\exp\pbra{-\frac{n\lambda}{8}}.
	\end{align}	
\end{lemma}

\begin{corollary}\label{cor:MatrixBounds}
	Let $X_1,\cdots,X_n$ be i.i.d. samples from $\calD$ where $X_k\in \R^{d\times d}$. Suppose $X_k$ is positive semi-definite, and $\normop{X_k}\le 1$ for all $k\in [n]$ almost surely. For any fixed $t>0$ we have
	\begin{align}
		&\Pr\( \forall m\le n,\sum_{k=1}^{m}X_k \mle 6tI+2m\E_{X\sim \calD}[X] \) \geq 1-2nd\exp\pbra{-t/4}.\label{equ:cor2-1}\\
		% &\Pr\( \forall m\le n,3tI+\sum_{k=1}^{m}X_k \mge \frac{m}{8}\E_{X\sim \calD}[X] \) \geq 1-2nd\exp\pbra{-t/8}.\label{equ:cor2-2}
	\end{align}
\end{corollary}
\begin{proof}
	For any fixed $m\le n$, applying \cref{lem:MatrixUpperBound} with $\lambda=t/m$ we get
	\begin{align}
		&\Pr\(\sum_{k=1}^{m}X_k \mle 6tI+2m\E_{X\sim \calD}[X] \) =\Pr\(\frac{1}{m}\sum_{k=1}^{m}X_k \mle 6\frac{t}{m}I+2\E_{X\sim \calD}[X] \) \\
		\le&\; 1-2d\exp\pbra{-t/4}.
	\end{align}
	By union bound over $m\in[n]$ we prove Eq.~\eqref{equ:cor2-1}.
	%Similarly, applying \cref{lem:MatrixLowerBound} and union bound we get Eq.~\eqref{equ:cor2-2}.
\end{proof}

\subsection{Relation Between Offline and Online Covariance Matrices}
\good{check that the values for $\lambda$ are indeed in the prescribed range $\lambda \in [0,1]$}
\textbf{Notation}: Let $n_k$ be the expected number of samples in the online phase allocated to policy $\pi^{k}$.
\begin{lemma}[Matrix Upper Bound Offline Phase]
\label{lem:SigmaOffline}
\cref{alg:Planner} produces a comulative covariance matrix $\Sigma_M$ that satisfies
\begin{align}
	\Pro\(\Sigma_M \preccurlyeq  2\overline \Sigma \) \geq 1-\frac{\delta}{4}
\end{align}
as long as
\begin{align}
	M \ge \frac{96KN}{\lambda_{reg}} \ln\frac{192dNK}{\lambda_{reg}\delta}.
\end{align}
\end{lemma}
\begin{proof}
From 
\fullref{lem:Switches} we know that at most $K$ distinct policies are produced during the execution of \cref{alg:Planner} where $K$ is defined in that lemma.

Let $\phi_i^{(k)}$ be the $i$ sampled feature during phase $k$. Let $m_k$ be the values that $m$ takes on during phase $k$. 
Note that although $m_k$ is a random variable, we have $m_k\le M$ almost surely. As a result, applying \cref{cor:MatrixBounds} with $t=\frac{\lambda_{reg}}{6\alpha K}$ we get
\begin{align}
\label{eqn:nkcov}
\Pro\(\sum_{i=1}^{m_{k}}\phi^{(k)}_i\phi^{(k),\top}_i \preccurlyeq \Big(\frac{\lambda_{reg}}{\alpha K} I + 2m_k\E_{\phi \sim d^{(k)}}[\phi\phi^\top]\Big) \) \geq  1-2dM\exp\pbra{-\frac{\lambda_{reg}}{24\alpha K}}.
\end{align}
Now multiplying the event inside the probability by $\alpha = \frac{n_k}{m_k} = \frac{N}{M}$ we get
\begin{align}
	\Sigma^{(k)} \defeq \alpha \sum_{i=1}^{m_{k}}\phi^{(k)}_i\phi^{(k),\top}_i \preccurlyeq \frac{\lambda_{reg}}{K}I + 2 n_k \E_{\phi \sim d^{(k)}}[\phi\phi^\top].
\end{align}
After a union bound on the number of phases we can write
\begin{align}
&\Pro\( \Sigma_M \defeq \sum_{k=1}^{K} \Sigma^{(k)} + \lambda_{reg} I \preccurlyeq 2\lambda_{reg} I + 2 \sum_{k=1}^K n_k \E_{\phi \sim d^{(k)}}[\phi\phi^\top] \defeq 2\overline \Sigma \) \\
\geq\; &1-2dMK\exp\pbra{-\frac{\lambda_{reg}}{24\alpha K}} .
\end{align}
and substituting the value of $\alpha$ on the right hand side now gives
\begin{align}
	\Pro\(\Sigma_M \preccurlyeq  2\overline \Sigma \) \geq 1-2dMK\exp\pbra{-\frac{\lambda_{reg}M}{24KN}}.
\end{align}
The final result follows from basic algebra.
\end{proof}

\begin{lemma}[Matrix Upper Bound Online Phase]
\label{lem:SigmaOnline}
Recall that
\begin{align}
	\Sigma'_{n} = \sum_{j=1}^{n-1} \phi(s'_j,a'_j)\phi(s'_j,a'_j)^\top + \lambda_{reg} I, \qquad \overline \Sigma  = \alpha \sum_{i=1}^M \E_{\phi \sim d_i} \phi\phi^\top + \lambda_{reg} I.
\end{align}
Algorithm~\ref{alg:Sampler} produces a cumulative covariance matrix $\Sigma'_N$ that satisfies
\begin{align}
		\Pro\( 9  
		\Sigma'_N \succcurlyeq \overline \Sigma \) \geq 1-\frac{\delta}{4}
\end{align}
as long as $\lambda_{reg}\ge 24\ln\frac{8d}{\delta}$.
\end{lemma}
\begin{proof}
	Notice that conditioned on the run of \cref{alg:Planner}, the distributions $d^{(1)},\dots,d^{(K)}$ of the features $\phi$ corresponding to the policies $\pi^{(1)},\dots,\pi^{(K)}$ are fixed (non-random), hence 
	$
	\overline \Sigma = \sum_{k=1}^K \Sigma^{(k)}
	$ is non-random. Also, let $\phi'_i$ be the $i$ feature vector collected during the online phase.  Notice that conditioned on all the random variables during the offline portion we can write
	\begin{align}
		\E \Sigma'_N & = \lambda_{reg}I + \E \sum_{i=1}^{N} (\phi'_i)(\phi'_i)^\top \\
		& = \lambda_{reg}I +  \sum_{k=1}^{K} \E \sum_{i=1}^{n'_k} \E_{\phi\sim d^{(k)}} \phi\phi^\top 
		\end{align}
		where $n'_k$ is the number of times that policy $\pi^{(k)}$ is sampled during the online phase. Continuing, 
		\begin{align}
		& = \lambda_{reg}I +  \sum_{k=1}^{K} \E n'_k \E_{\phi\sim d^{(k)}} \phi\phi^\top \\
		& =  \lambda_{reg}I + \sum_{k=1}^{K} n_k \E_{\phi\sim d^{(k)}} \phi\phi^\top \\
		%& \succcurlyeq  \sum_{k=1}^{K} \alpha m_k \E_{\phi\sim d^{(k)}} \phi\phi^\top \\
		& \defeq \overline \Sigma.
\end{align}

Now, apply \fullref{lem:MatrixLowerBound} with $\lambda = \lambda_{reg}/(3N)$ we get
\begin{align}\label{equ:lem10-1}
	&\Pro\(8\(3N\lambda I + \sum_{i=1}^{N} (\phi'_i)(\phi'_i)^\top\)  \succcurlyeq \sum_{k=1}^{K} n_k \E_{\phi\sim d^{(k)}} \phi\phi^\top\) \\
	\geq\; &1-2d\exp\pbra{-\frac{N\lambda}{8}}= 1-2d\exp\pbra{-\frac{\lambda_{reg}}{24}}.
\end{align}
Recall that $\Sigma_N'=\lambda_{reg} I + \sum_{i=1}^{N} (\phi'_i)(\phi'_i)^\top$ and $\overline{\Sigma}=\lambda_{reg} I +\sum_{k=1}^{K} n_k \E_{\phi\sim d^{(k)}} \phi\phi^\top$. 
Eq.~\eqref{equ:lem10-1} implies that
\begin{align}
&\Pro\( 3N\lambda I + 8\(\underbrace{\underbrace{3N\lambda}_{\lambda_{reg}} I + \sum_{i=1}^{N} (\phi'_i)(\phi'_i)^\top}_{\Sigma'_{N}}\)  \succcurlyeq \underbrace{3N\lambda I +\sum_{k=1}^{K} n_k \E_{\phi\sim d^{(k)}} \phi\phi^\top}_{\overline \Sigma} \)\\
 \geq&\; 1-2d\exp\pbra{-\frac{\lambda_{reg}}{24}}.
\end{align}
By setting $\lambda_{reg} \ge 24\ln\frac{8d}{\delta}$ we have
\begin{align}
	\Pro\(9\Sigma_N'\succcurlyeq \overline \Sigma\)\ge 1-\delta/4.
\end{align}
\end{proof}

\section{Scalar Concentration Inequalities for Martingales}
\subsection{Bernstein Inequality for Martingales}
The following lemma is the same as Theorem 1 from \cite{beygelzimer2011contextual} as is reported here for completeness.
\begin{theorem}[Bernstein's Inequality for Martingales]
\label{lem:Bernstein}
Consider the stochastic process $\{X_t\}$ adapted to the filtration $\{ \F_t \}$. Assume $X_t \leq 1$ almost surely. Then 
\begin{align}
\forall \lambda \in (0,1], \qquad \Pro\( \sum_{t=1}^T X_t \leq  \lambda \sum_{t=1}^T \E_t X^2_t +  \frac{1}{\lambda}\ln\frac{1}{\delta} \) \geq 1-\delta,
\end{align}
which implies
\begin{align}
	\Pro\( \sum_{t=1}^T X_t \leq  2\sqrt{\(\sum_{t=1}^T \E_t X^2_t \) \ln\frac{1}{\delta}} +  2\ln\frac{1}{\delta} \) \geq 1-\delta.
\end{align}
\end{theorem}
For completeness, we reprove the theorem below:
\begin{proof}
Define the random variable $M_t$ as
\begin{align}
M_t = M_{t-1}\exp(\lambda X_t - \lambda^2 \E_t X^2_t)
\end{align}
where in particular $M_0 = 1$. Recall the inequality $e^x \leq 1+x + x^2$ for $x \leq 1$ and $1+x \leq e^x$:
\begin{align}
\E_t M_t & = M_{t-1} \E_t \exp(\lambda X_t - \lambda^2 \E_t X^2_t) \\
& \leq M_{t-1} \E_t \Big[(1+\lambda X_t + \lambda^2 X_t^2)\Big]\exp(-\lambda^2 \E_t X^2_t) \\
& \leq M_{t-1} \exp(\lambda^2 \E_t X_t^2)\exp(-\lambda^2 \E_t X^2_t) \\
& = M_{t-1}.
\end{align}
Thus $\{ M_t \}$ is a supermartingale sequence adapted to $\{ \F_t\}$. In particular, $\E M_t \leq M_0 = 1$ using the tower property.
Now by the Markov inequality
\begin{align}
	\Pro\( \underbrace{\lambda \sum_{t=1}^T X_t - \lambda^2\sum_{t=1}^T \E_t X_t^2}_{\ln M_t} > \ln\frac{1}{\delta}\) = \Pro\(M_t >  \frac{1}{\delta}\)
    \leq \frac{\E M_t}{\frac{1}{\delta}} \leq \delta
%> \delta \E M_t \leq 1.
\end{align}
This implies that with probability at least $1-\delta$ the following event holds:
\begin{align}
\lambda \sum_{t=1}^T X_t - \lambda^2 \sum_{t=1}^T \E_t X^2_t = \ln M_t \leq \ln\frac{1}{\delta}
\end{align}
which is the first part of the thesis. Now, we choose $\lambda$. If $\sum_{t=1}^T \E_t X_t^2 \leq \ln \frac{1}{\delta}$ then under the above event we obtain with $\lambda = 1$ (the largest possible value)
\begin{align}
\label{eqn:b1}
	\sum_{t=1}^T X_t \leq \sum_{t=1}^T \E_t X^2_t + \ln\frac{1}{\delta} \leq  2\ln\frac{1}{\delta}.
\end{align}
If conversely $\sum_{t=1}^T \E_t X_t^2 \geq \ln \frac{1}{\delta}$ then let $\lambda = \sqrt{\frac{\ln\frac{1}{\delta}}{\sum_{t=1}^T \E_t X_t^2 }} \leq 1$ to obtain (still under the same event)
\begin{align}
\label{eqn:b2}
	\sum_{t=1}^T X_t \leq 2\sqrt{\(\sum_{t=1}^T \E_t X^2_t \) \ln \frac{1}{\delta}}.
\end{align}
Therefore, summing the rhs of \cref{eqn:b1,eqn:b2} to cover both cases we obtain the second part of the thesis.
\end{proof}

\subsection{Reverse Bernstein Inequality for Martingales}
Bernstein's inequality bounds the sum a random variable $\sum_{t} X_t$ using second moment information $\sum_{t} \Var_t X_t$; in our case (positive random variables in $[0,1]$), the sum of the conditional variances $\sum_{t} \Var_t X_t$ is upper bounded by the sum of the means $\sum_{t} \E_t X_t$. 

This section provides the `reverse' inequality: Assuming a bound on the sum a random variable $\sum_{t} X_t$, it bounds the conditional sum of the means $\sum_{t} \E_t X_t$.
\begin{theorem}[Reverse Bernstein for Martingales]
\label{thm:ReverseBernstein}
Let $(\Sigma,\F,\Pro)$ be a probability space and consider the stochastic process $\{X_t\}$ adapted to the filtration $\{ \F_t \}$. Let $\E_t X_t \defeq \E[X_t \mid \F_{t-1}]$ be the conditional expectation of $X_t$ given $\F_t$. If $0 \leq X_t \leq 1$ then it holds that
\begin{align}
\Pro \(\sum_{t=1}^T \E_t X_t \geq \frac{1}{4}\( c_1 + \sqrt{c_1^2+4\(\sum_{t=1}^T X_t+c_2 \)}\)^2
\) \leq \delta, \qquad \qquad c_1 = 2\sqrt{\ln\frac{1}{\delta}}, \; c_2 = 2\ln\frac{1}{\delta}
\end{align}
\end{theorem}
\begin{proof}
Consider the random `noise'
\begin{align}
\xi_t \defeq \E_t X_t - X_t
\end{align}
which allows us to write
\begin{align}
\label{eqn:thissum}
\sum_{t=1}^T \E_t X_t = \sum_t^T (\xi_t + X_t)
\end{align}
Then \fullref{lem:Bernstein} ensures the following high probability statement for appropriate $c_1 = 2\sqrt{\ln\frac{1}{\delta}}$, $c_2 = 2\ln\frac{1}{\delta}$:
\begin{align}
\Pro\( \sum_{t=1}^T \xi_t  \leq c_1\sqrt{\sum_{t=1}^T{\E}_t \xi^2_t} + c_2 \) \geq 1-\delta.
\end{align}
Notice that since $0 \leq X_t \leq 1$ we have 
\begin{align}\E_t \xi^2_t & =  \E_t (X_t - \E_t X_t)^2 \\
& = \E_t X^2_t - (\E_t X_t)^2 \\
& \leq \E_t X^2_t \\
& \leq \E_t X_t.
\end{align}
Plugging back into the above display and using \cref{eqn:thissum} gives
\begin{align}
\Pro\(\sum_{t=1}^T \xi_t  = \sum_{t=1}^T \( \E_t X_t - X_t  \) \leq c_1\sqrt{\sum_{t=1}^T \E_tX_t} + c_2 \) \geq 1-\delta
\end{align}
or equivalently
\begin{align}
\Pro\( \sum_{t=1}^T \E_t X_t \leq \sum_{t=1}^T X_t  + c_1\sqrt{\sum_{t=1}^T\E_t X_t} + c_2 \) \geq 1-\delta.
\end{align}
Solving for $\sum_{t=1}^T \E_t X_t$ gives under such event
\begin{align}
\sum_{t=1}^T \E_t X_t \leq \frac{1}{4}\( c_1 + \sqrt{c_1^2+4\(\sum_{t=1}^T X_t+c_2\)}\)^2.
\end{align} 
\end{proof}

\section{Additional Remarks}
\subsection{Hard Instance that Requires Exploration}
\label{sec:HardInstance}
Even with a fixed context, uniform exploration ignores the structure of the reward function, and is therefore suboptimal. Consider the case where $d=2$ and $\phi(1,1)=(1,0)$ and $\phi(1,i)=(0,1)$ for $2\le i\le A.$ The context is fixed and there are $A$ actions. A uniform exploration requires $\Omega(A/\epsilon^2)$ samples to estimate the reward for action $1$. In contrast, the optimal exploration policy is $$\pi(1)=\begin{cases}1,&\text{ w.p. 1/2},\\2,&\text{ w.p. 1/2}.\end{cases}$$ And the corresponding sample complexity is $\tildeO(1/\epsilon^2).$

On the other hand, the policy that ignores context is also suboptimal. Consider the policy $\pi^G(s)$ that returns the G-optimal design on the action set $\calA_s.$ The policy $\pi^G$ explores optimally for a fixed context.  \citet{ruan2020linear}'s Lemma 4 implies that this policy achieves an online sample complexity $\tildeO(d^3/\epsilon^2).$ For completeness, we also include a hard instance for $\pi^G.$ Let $\calS=\{s_1,\cdots,s_k\}$ and $\calA_s=\{a_1,\cdots,a_{k+1}\}.$ Assume a uniform distribution over the state space $\calS$. The feature vector is defined as follows.
\begin{align}
	\phi(s_i,a_j)=\begin{cases}
		e_j,&\text{ when }j\le k,\\
		e_{i+k},&\text{ when }j=k+1.
	\end{cases}
\end{align}
Note that in this case, the dimension of the feature vector is $d=2k$. For a fixed $s\in\calS$, the G-optimal design returns the uniform exploration policy. As a result, the expected covariance matrix is $\Sigma=\E_{s\sim\calS,a\sim \calA}\phi(s,a)\phi(s,a)^\top=\diag(1/(k+1),\cdots,1/(k+1),1/k(k+1),\cdots,1/k(k+1)).$ It follows that
\begin{align*}
	\E_{s\sim \calS}\max_{a\in\calA_s}\phi(s,a)^\top \Sigma^{-1} \phi(s,a)\ge \E_{s\sim \calS}\phi(s,a_{k+1})^\top \Sigma^{-1} \phi(s,a_{k+1})=k(k+1)=\Omega(d^2).
\end{align*}
In contrast, the optimal exploration policy is 
\begin{align}
	\pi(s_i)=\begin{cases}
		U(\{a_1,\cdots,a_k\}),&\text{ w.p. }1/2,\\
		a_{k+1},&\text{ w.p. }1/2,
	\end{cases}
\end{align}
where $U(\cdot)$ denotes the uniform distribution. Correspondingly, we have
$\Sigma=\diag(1/2k,\cdots,1/2k),$ and
\begin{align}
	\E_{s\sim \calS}\max_{a\in\calA_s}\phi(s,a)^\top \Sigma^{-1} \phi(s,a)\le O(d).
\end{align}

\subsection{Remarks Regarding \citet{ruan2020linear}}
\label{sec:Ruan}
Section 6 of \citet{ruan2020linear} argues that when $\lambda_{reg}<1/d$, the covariance matrix doesn't concentrate. Their construction works as follows. Consider a fixed offline dataset $\calC$ with size $M$. The context space is $\calS=[d]$. The action space is $\calA_1=\{1\}$ and $\calA_s=\{1,2\}$ for $2\le s\le d.$ The feature vector is defined as
\begin{align}
	\phi(1,1)=e_1,\quad \phi(s,1)=e_s,\quad \phi(s,2)=\sqrt{1-\frac{d}{M}}e_s+\sqrt{\frac{d}{M}}e_1.
\end{align}
The context distribution $\mu$ is 
\begin{align}
	\mu(1)=\frac{1}{dM},\quad \mu(s)=\frac{1}{d-1}\pbra{1-\frac{1}{dM}},\forall s\ge 2.
\end{align}
Then with probability at least $1/d$, $\calC$ contains a single occurrence of context $1$. \citet{ruan2020linear} argues that in this case, there exits a policy $\pi$ such that the covariance on $\calC$ deviates from the population one when $\lambda_{reg}<\frac{1}{d}$.

The policy is $\pi(s)=1.$ Let $\hat\Sigma=\lambda_{reg}I+\sum_{s_i\in \calC}\phi(s_i,\pi(s_i))\phi(s_i,\pi(s_i))^\top.$ We can compute that
\begin{align}
	\E_{s\sim \calC}\max_{a\in\calA_s} \phi(s,a)^\top \hat\Sigma^{-1} \phi(s,a)\le O(d/M).
\end{align}
Now consider the true distribution $\mu$. Let $\Sigma=\lambda_{reg}I+M\E_{s\sim \mu}\phi(s,\pi(s))\phi(s,\pi(s))^\top.$ By basic algebra we get $\Sigma=\lambda_{reg}I+M\diag(1/(dM),\mu(2),\cdots,\mu(d)).$
As a result, we can compute 
\begin{align*}
	&\E_{s\sim \mu}\max_{a\in\calA_s} \phi(s,a)^\top \Sigma^{-1} \phi(s,a)\\
	=&\;\mu(1)\frac{d}{2}+\sum_{s=2}^{d}\mu(s)\maxtwo{\frac{1}{\lambda_{reg}+M\mu(s)}}{\pbra{1-\frac{d}{M}}\frac{1}{\lambda_{reg}+M\mu(s)}+\frac{d}{M}\frac{1}{\lambda_{reg}+1/d}}\\
	\ge&\;\Omega(d^2/M).
\end{align*}
Note the inequality follows due to the rightmost term and substituting in $\lambda_{reg} < \frac{1}{d}.$ Comparing the estimate of the empirical covariance with the true expectation, we observe the covariance matrix doesn't concentrate. Indeed, in the setting where $\lambda_{reg}<1$, the concentration events $\calE_1,\calE_2,\calE_3$ in the proof of Theorem~\ref{thm:main} fail with constant probability. Since \citet{ruan2020linear} focus on small regularization setting, their algorithm first finds a ``core'' of the contexts. This procedure makes their algorithm much more complex compared with ours, and increases the number of offline samples required.

Now if we set $\lambda_{reg}=1$, by the same computation we get
\begin{align*}
	&\E_{s\sim \mu}\max_{a\in\calA_s} \phi(s,a)^\top \Sigma^{-1} \phi(s,a)\le \Omega(d/M).
\end{align*}
Hence, the concentration events hold and Theorem~\ref{thm:main} gives a much tighter offline sample complexity.

%% file: contents/9-appendix-helper.tex
\section{Helper Lemmas}
\begin{lemma}[Number of Switches]
\label[lemma]{lem:Switches}
Algorithm \ref{alg:Planner} generates at most $K$ distinct policies:
\begin{align}
\label{eqn:Kdef}
K \leq d\ln_2 \(1 + \frac{M}{d\lambda_{reg}}\) = \widetilde O(d)
\end{align}
\end{lemma}
\begin{proof}
Notice that $\det(\Sigma_1) = \lambda_{reg}^d$ and $\det(\Sigma_M) \leq (\lambda + \frac{M}{d})^d$ (see proof of lemma 11 in \citep{Abbasi11}). Every time the policy changes the determinant of $\Sigma_m$ at least doubles.
Let $K$ denote the number of times the policy is updated.
By induction,
\begin{align}
\(\lambda_{reg} + \frac{M}{d}\)^d \geq \det(\Sigma_M) \geq 2^K \det(\Sigma_1) \geq 2^K \lambda_{reg}^d
\end{align}
Solving for $K$ concludes.
\end{proof}

%% file: contents/9-appendix-exp.tex
\section{Additional Experiments and Information}

The Yahoo! dataset is available freely for research purposes through Yahoo! Webscope. Use of the dataset required accepting an agreement not to share the original in a way that the dataset could be reconstructed. The dataset is available at the following link:

 \href{https://webscope.sandbox.yahoo.com/catalog.php?datatype=c}{https://webscope.sandbox.yahoo.com/catalog.php?datatype=c}

The data consists entirely of numerical feature vectors and does not contain any identifiable or offensive information. To run the experiments, we used a standard Amazon Web Services EC2 c5.xlarge instance with 4 vCPUs and 8gb of memory.

For the subsampling to create $300$-dimensional feature vectors, we selected the following $0$-indexed coordinates to include out of full $700$. These coordinates were chosen randomly by sampling with replacement. The specific indices that were used are in the supplementary material.

\subsection{Additional Plots}

In addition to the three choices of regularization shown in Figure~\ref{fig::real-reg} for the real-world dataset, we ran additional more extreme values, but omitted them for clarity in the plot. In Figure~\ref{fig::add-reg}, we see the same algorithms plotted with both large values of regularization ($\geq 1$) and small values of regularization ($\leq 1$).

\begin{figure}
\begin{subfigure}{.5\textwidth}
\centering
\includegraphics[width=7cm]{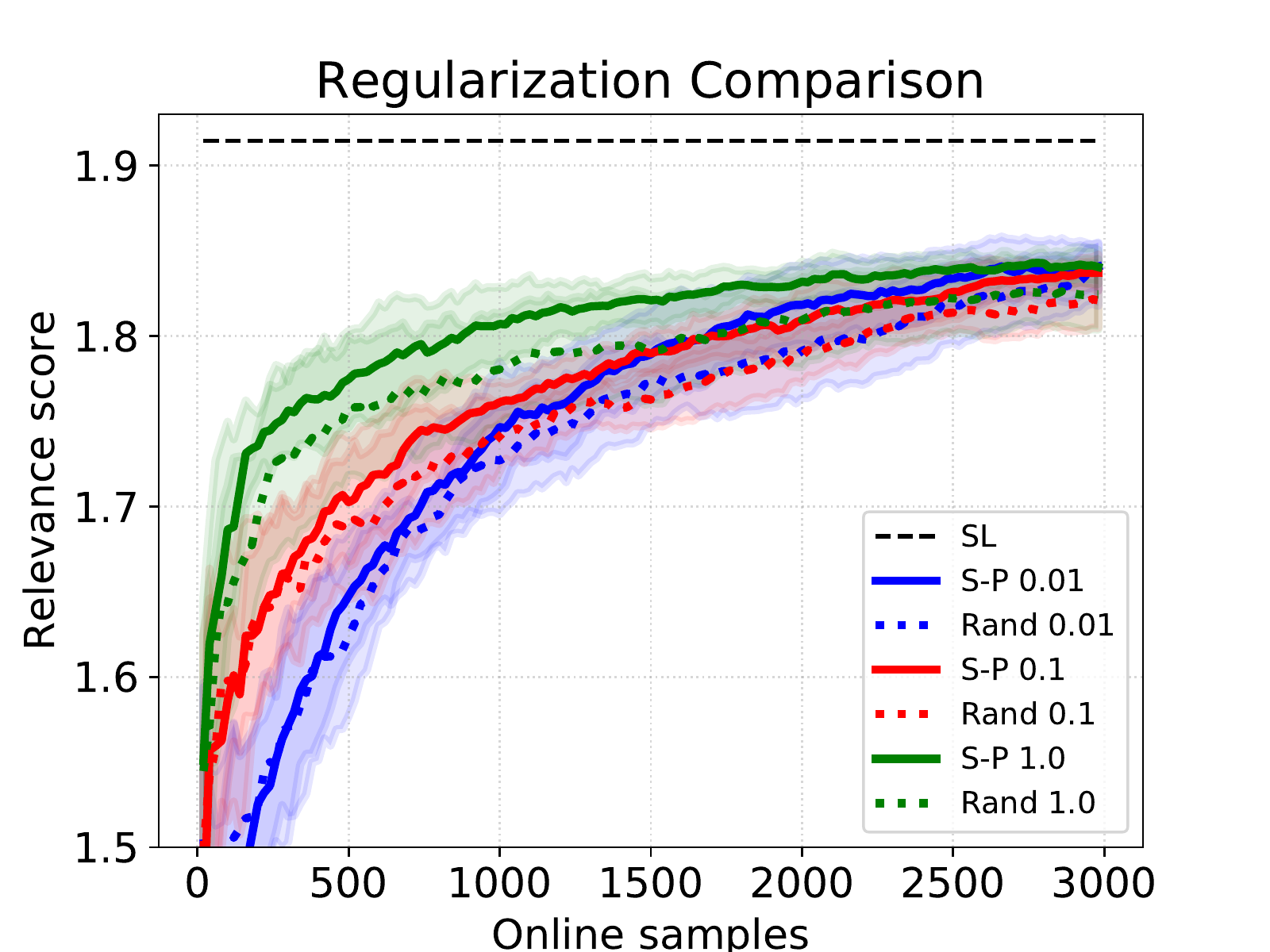} 
\caption{}
\end{subfigure}
\begin{subfigure}{.5\textwidth}
\centering
\includegraphics[width=7cm]{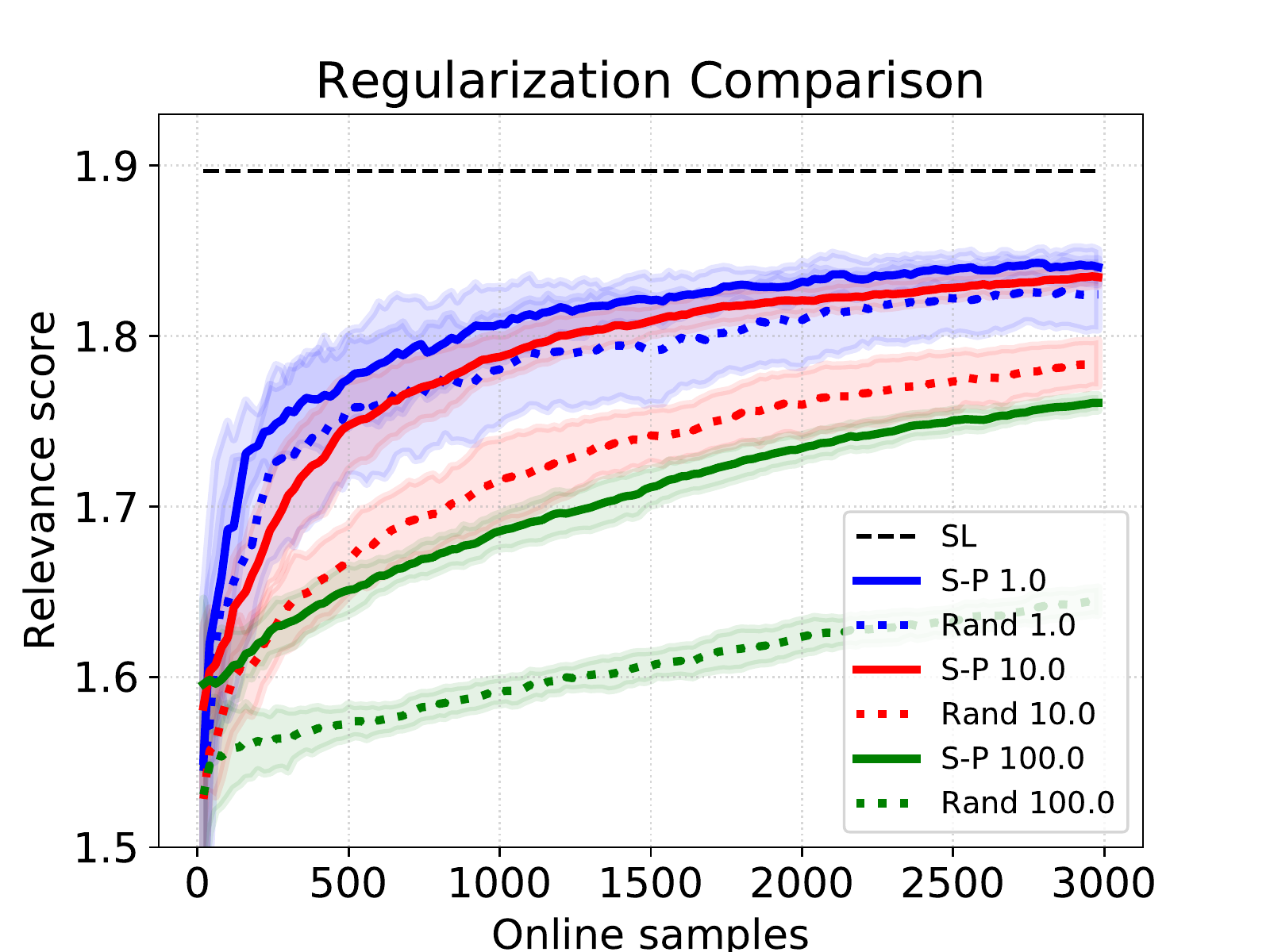}
\caption{}
\end{subfigure}%
    \caption{ The left figure shows the sampler-planner (S-P) compared to the random algorithm with small regularization $\lambda \in \{ 0.01, 0.1, 1.0\}$. The right shows the same for $\lambda \in \{1.0, 10.0, 100.0\}$. } \label{fig::add-reg}
    \vspace{-3mm}
\end{figure}

%% file: DoE.bbl
\begin{thebibliography}{}

\bibitem[Abbasi-Yadkori et~al., 2011a]{Abbasi11}
Abbasi-Yadkori, Y., Pal, D., and Szepesvari, C. (2011a).
\newblock Improved algorithms for linear stochastic bandits.
\newblock In {\em Advances in Neural Information Processing Systems (NIPS)}.

\bibitem[Abbasi-Yadkori et~al., 2011b]{abbasi2011improved}
Abbasi-Yadkori, Y., P{\'a}l, D., and Szepesv{\'a}ri, C. (2011b).
\newblock Improved algorithms for linear stochastic bandits.
\newblock In {\em NIPS}, volume~11, pages 2312--2320.

\bibitem[Bai et~al., 2019]{bai2019provably}
Bai, Y., Xie, T., Jiang, N., and Wang, Y.~X. (2019).
\newblock Provably efficient q-learning with low switching cost.
\newblock {\em Advances in Neural Information Processing Systems}, 32.

\bibitem[Beygelzimer et~al., 2011]{beygelzimer2011contextual}
Beygelzimer, A., Langford, J., Li, L., Reyzin, L., and Schapire, R. (2011).
\newblock Contextual bandit algorithms with supervised learning guarantees.
\newblock In {\em Proceedings of the Fourteenth International Conference on
  Artificial Intelligence and Statistics}, pages 19--26. JMLR Workshop and
  Conference Proceedings.

\bibitem[Chapelle and Chang, 2011]{chapelle2011yahoo}
Chapelle, O. and Chang, Y. (2011).
\newblock Yahoo! learning to rank challenge overview.
\newblock In {\em Proceedings of the learning to rank challenge}, pages 1--24.
  PMLR.

\bibitem[Chu et~al., 2011]{chu2011contextual}
Chu, W., Li, L., Reyzin, L., and Schapire, R. (2011).
\newblock Contextual bandits with linear payoff functions.
\newblock In {\em Proceedings of the Fourteenth International Conference on
  Artificial Intelligence and Statistics}, pages 208--214. JMLR Workshop and
  Conference Proceedings.

\bibitem[Cortes et~al., 2019]{cortes2019behavioral}
Cortes, K.~E., Fricke, H.~D., Loeb, S., Song, D.~S., and York, B.~N. (2019).
\newblock When behavioral barriers are too high or low--how timing matters for
  parenting interventions.
\newblock Technical report, National Bureau of Economic Research.

\bibitem[Degenne et~al., 2020]{degenne2020gamification}
Degenne, R., M{\'e}nard, P., Shang, X., and Valko, M. (2020).
\newblock Gamification of pure exploration for linear bandits.
\newblock In {\em International Conference on Machine Learning}, pages
  2432--2442. PMLR.

\bibitem[Deshmukh et~al., 2018]{deshmukh2018simple}
Deshmukh, A.~A., Sharma, S., Cutler, J.~W., Moldwin, M., and Scott, C. (2018).
\newblock Simple regret minimization for contextual bandits.
\newblock {\em arXiv preprint arXiv:1810.07371}.

\bibitem[Doss et~al., 2019]{doss2019more}
Doss, C., Fahle, E.~M., Loeb, S., and York, B.~N. (2019).
\newblock More than just a nudge supporting kindergarten parents with
  differentiated and personalized text messages.
\newblock {\em Journal of Human Resources}, 54(3):567--603.

\bibitem[Esfandiari et~al., 2019]{esfandiari2019regret}
Esfandiari, H., Karbasi, A., Mehrabian, A., and Mirrokni, V. (2019).
\newblock Regret bounds for batched bandits.
\newblock {\em arXiv preprint arXiv:1910.04959}.

\bibitem[Fiez et~al., 2019]{fiez2019sequential}
Fiez, T., Jain, L., Jamieson, K., and Ratliff, L. (2019).
\newblock Sequential experimental design for transductive linear bandits.
\newblock {\em arXiv preprint arXiv:1906.08399}.

\bibitem[Foster et~al., 2018]{foster2018practical}
Foster, D., Agarwal, A., Dudik, M., Luo, H., and Schapire, R. (2018).
\newblock Practical contextual bandits with regression oracles.
\newblock In {\em International Conference on Machine Learning}, pages
  1539--1548. PMLR.

\bibitem[Han et~al., 2020]{han2020sequential}
Han, Y., Zhou, Z., Zhou, Z., Blanchet, J., Glynn, P.~W., and Ye, Y. (2020).
\newblock Sequential batch learning in finite-action linear contextual bandits.
\newblock {\em arXiv preprint arXiv:2004.06321}.

\bibitem[Jedra and Proutiere, 2020]{jedra2020optimal}
Jedra, Y. and Proutiere, A. (2020).
\newblock Optimal best-arm identification in linear bandits.
\newblock {\em Advances in Neural Information Processing Systems}, 33.

\bibitem[Kiefer and Wolfowitz, 1960]{kiefer1960equivalence}
Kiefer, J. and Wolfowitz, J. (1960).
\newblock The equivalence of two extremum problems.
\newblock {\em Canadian Journal of Mathematics}, 12:363--366.

\bibitem[Lattimore and Szepesv{\'a}ri, 2020]{lattimore2020bandit}
Lattimore, T. and Szepesv{\'a}ri, C. (2020).
\newblock {\em Bandit Algorithms}.
\newblock Cambridge University Press.

\bibitem[Lattimore and Szepesvari, 2020]{lattimore2020learning}
Lattimore, T. and Szepesvari, C. (2020).
\newblock Learning with good feature representations in bandits and in rl with
  a generative model.
\newblock In {\em International Conference on Machine Learning (ICML)}.

\bibitem[Ren et~al., 2020]{ren2020batched}
Ren, Z., Zhou, Z., and Kalagnanam, J.~R. (2020).
\newblock Batched learning in generalized linear contextual bandits with
  general decision sets.
\newblock {\em IEEE Control Systems Letters}.

\bibitem[Ruan et~al., 2020]{ruan2020linear}
Ruan, Y., Yang, J., and Zhou, Y. (2020).
\newblock Linear bandits with limited adaptivity and learning distributional
  optimal design.
\newblock {\em arXiv preprint arXiv:2007.01980}.

\bibitem[Soare et~al., 2014]{soare2014best}
Soare, M., Lazaric, A., and Munos, R. (2014).
\newblock Best-arm identification in linear bandits.
\newblock In {\em Advances in Neural Information Processing Systems (NIPS)},
  pages 828--836.

\bibitem[Tao et~al., 2018]{tao2018best}
Tao, C., Blanco, S., and Zhou, Y. (2018).
\newblock Best arm identification in linear bandits with linear dimension
  dependency.
\newblock In {\em International Conference on Machine Learning}, pages
  4877--4886. PMLR.

\bibitem[Tropp, 2012]{tropp2012user}
Tropp, J.~A. (2012).
\newblock User-friendly tail bounds for sums of random matrices.
\newblock {\em Foundations of computational mathematics}, 12(4):389--434.

\bibitem[Wang et~al., 2021]{wang2021provably}
Wang, T., Zhou, D., and Gu, Q. (2021).
\newblock Provably efficient reinforcement learning with linear function
  approximation under adaptivity constraints.
\newblock {\em arXiv preprint arXiv:2101.02195}.

\bibitem[Xu et~al., 2018]{xu2018fully}
Xu, L., Honda, J., and Sugiyama, M. (2018).
\newblock A fully adaptive algorithm for pure exploration in linear bandits.
\newblock In {\em International Conference on Artificial Intelligence and
  Statistics}, pages 843--851. PMLR.

\bibitem[Zanette et~al., 2021]{zanette2021cautiously}
Zanette, A., Cheng, C.-A., and Agarwal, A. (2021).
\newblock Cautiously optimistic policy optimization and exploration with linear
  function approximation.
\newblock {\em arXiv preprint arXiv:2103.12923}.

\end{thebibliography}
